\documentclass{article}

\usepackage[dvips]{graphicx} 
\usepackage{subcaption} 
\usepackage[square,sort,comma,numbers]{natbib}
\usepackage{xr}
\usepackage{xcolor}

\usepackage{algorithm,algorithmic}
\usepackage{fullpage}
\usepackage{hyperref}
\usepackage{mathtools}
\usepackage{color}
\usepackage{amsthm}
\usepackage{amssymb}
\usepackage{amsopn}
\usepackage{mathtools}
\usepackage{amsthm}
\usepackage{xfrac}

\newcommand{\Rad}{\mathcal{R}}

\newcommand{\Dcal}{\mathcal{D}}

\newcommand{\Hcal}{\mathcal{H}}
\newcommand{\Mcal}{\mathcal{M}}
\newcommand{\Regret}{\mathrm{Regret}}

\DeclareMathOperator{\Tr}{\mathrm{Tr}}

\newtheorem{innercustomthm}{Theorem}
\newenvironment{customthm}[1]
  {\renewcommand\theinnercustomthm{#1}\innercustomthm}
  {\endinnercustomthm}

\def\norm#1{\mathopen\| #1 \mathclose\|}

\newcommand{\ignore}[1]{}

\def\bold0{\mathbf{0}}

\def\eps{\varepsilon}
\def\epsilon{\varepsilon}

\newcommand{\defeq}{\stackrel{\text{def}}{=}}

\newcommand{\pa}[1]{\left(#1\right)}
\newcommand{\ang}[1]{\left<#1\right>}
\newcommand{\bra}[1]{\left[#1\right]}
\newcommand{\abs}[1]{\left|#1\right|}

\newcommand{\mat}[1]{\begin{matrix}#1\end{matrix}}

\newcommand{\bmat}[1]{\bra{\mat{#1}}}

\DeclareMathOperator{\argmin}{argmin}

\DeclareMathOperator*{\E}{\mathbb{E}}

\newcommand{\R}{\mathbb{R}}

\newtheorem{theorem}{Theorem}[section]
\newtheorem{corollary}[theorem]{Corollary}
\newtheorem{definition}[theorem]{Definition}
\newtheorem{lemma}[theorem]{Lemma}
\newtheorem{proposition}[theorem]{Proposition}

\newcommand{\punt}[1]{}

\title{Learning Linear Dynamical Systems via Spectral Filtering}
\author{Elad Hazan \qquad Karan Singh \qquad Cyril Zhang \\
\\
Department of Computer Science\\
Princeton University\\
Princeton, NJ 08540 \\
\texttt{\{ehazan, karans, cyril.zhang\}@cs.princeton.edu} \\
}

\begin{document} 

\maketitle


\begin{abstract}
We present an efficient and practical algorithm for the online prediction of discrete-time linear dynamical systems with a symmetric transition matrix. We circumvent the non-convex optimization problem using improper learning: carefully overparameterize the class of LDSs by a polylogarithmic factor, in exchange for convexity of the loss functions.
From this arises a polynomial-time algorithm with a near-optimal regret guarantee, with an analogous sample complexity bound for agnostic learning.
Our algorithm is based on a novel filtering technique, which may be of independent interest: we convolve the time series with the eigenvectors of a certain Hankel matrix. 

\end{abstract}


\section{Introduction}

Linear dynamical systems (LDSs) are a class of state space models which accurately model many phenomena in nature and engineering, and are applied ubiquitously in time-series analysis, robotics, econometrics, medicine, and meteorology. In this model, the time evolution of a system is explained by a linear map on a finite-dimensional hidden state, subject to disturbances from input and noise. Recent interest has focused on the effectiveness of recurrent neural networks (RNNs), a nonlinear variant of this idea, for modeling sequences such as audio signals and natural language.

Central to this field of study is the problem of \emph{system identification}: given some sample trajectories, output the parameters for an LDS which generalize to predict unseen future data. Viewed directly, this is a non-convex optimization problem, for which efficient algorithms with theoretical guarantees are very difficult to obtain. A standard heuristic for this problem is expectation-maximization (EM), which can find poor local optima in theory and practice.  

We consider a different approach: we formulate system identification as an online learning problem, in which neither the data nor predictions are assumed to arise from an LDS. Furthermore, we slightly overparameterize the class of predictors, yielding an online convex program amenable to efficient regret minimization. This carefully chosen relaxation, which is our main theoretical contribution, expands the dimension of the hypothesis class by only a polylogarithmic factor. This construction relies upon recent work on the spectral theory of Hankel matrices.

The result is a simple and practical algorithm for time-series prediction, which deviates significantly from existing methods.
We coin the term {\it wave-filtering} for our method, in reference to our relaxation's use of convolution by wave-shaped eigenvectors.
We present experimental evidence on both toy data and a physical simulation, showing our method to be competitive in terms of predictive performance, more stable, and significantly faster than existing algorithms.

\subsection{Our contributions}

Consider a discrete-time linear dynamical system with inputs $\{ x_t \}$, outputs $\{ y_t \}$, and a latent state $\{ h_t \}$, which can all be multi-dimensional.
With noise vectors $\{\eta_t\}, \{\xi_t\}$, the system's time evolution is governed by the following equations:
\begin{align*}
h_{t+1} &= A h_t + B x_t + \eta_t \\
y_t &= C h_t + D x_t + \xi_t.
\end{align*}
If the dynamics $A,B,C,D$ are known, then the Kalman filter \cite{kalman1960new} is known to estimate the hidden state optimally under Gaussian noise, thereby producing optimal predictions of the system's response
to any given input. However, this is rarely the case -- indeed, real-world systems are seldom purely linear, and rarely are their evolution matrices known. 

We henceforth give a provable, efficient algorithm for the prediction of sequences arising from an unknown dynamical system as above, in which the matrix $A$ is symmetric. Our main theoretical contribution is a regret bound for this algorithm,
giving nearly-optimal convergence to the lowest mean squared prediction error (MSE)
realizable by a symmetric LDS model:

\begin{customthm}{1}[Main regret bound; informal]
\label{thm:main-online}
On an arbitrary sequence $\{ (x_t, y_t) \}_{t=1}^T$, Algorithm~\ref{alg:wave-ogd} makes predictions $\{ \hat y_t \}_{t=1}^T$ which satisfy
\[ \mathrm{MSE}(\hat y_1, \ldots, \hat y_T) - \mathrm{MSE}(\hat y_1^*, \ldots, \hat y_T^*) \;\leq\; \tilde O\pa{ \frac{\mathrm{poly}(n,m,d,\log T)}{\sqrt{T}} },\]
compared to the best predictions $\{y_t^*\}_{t=1}^T$ by a symmetric LDS, while running in polynomial time.
\end{customthm}

Note that the signal need not be generated by an LDS, and can even be \emph{adversarially} chosen. In the less general batch (statistical) setting, we use the same techniques to obtain an analogous sample complexity bound for agnostic learning:
\begin{customthm}{2}[Batch version; informal]
\label{thm:main-batch}
For any choice of $\eps > 0$,
given access to an arbitrary distribution $\Dcal$
over training sequences $\{ (x_t, y_t) \}_{t=1}^T$,
Algorithm~\ref{alg:wave-batch}, run on $N$ i.i.d. sample trajectories from $\Dcal$,
outputs a predictor $\hat \Theta$ such that
\[ \E_\Dcal \bigg[ \mathrm{MSE}( \hat \Theta ) - \mathrm{MSE}( \Theta^* ) \bigg] \;\leq\; \eps + \frac{ \tilde O\pa{ \mathrm{poly}(n,m,d, \log T, \log 1/\eps) } }{\sqrt{N}},
\]
compared to the best symmetric LDS predictor $\Theta^*$, while running in polynomial time.
\end{customthm}

Typical regression-based methods require the LDS to be \emph{strictly} stable, and degrade on ill-conditioned systems; they depend on a spectral radius parameter $\frac{1}{1-\norm{A}}$. Our proposed method of \emph{wave-filtering} provably and empirically works even for the hardest case of $\norm{A} = 1$. Our algorithm attains the first condition number-independent polynomial guarantees in terms of regret (equivalently, sample complexity) and running time for the MIMO setting. Interestingly, our algorithms never need to learn the hidden state, and our guarantees can be sharpened to handle the case when the dimensionality of $h_t$ is infinite.



\subsection{Related work}
The modern setting for LDS arose in the seminal work of Kalman \cite{kalman1960new}, who introduced the Kalman filter as a recursive least-squares solution for maximum likelihood estimation (MLE) of Gaussian perturbations to the system. The framework and filtering algorithm have proven to be a mainstay in control theory and time-series analysis; indeed, the term \emph{Kalman filter model} is often used interchangeably with LDS. We refer the reader to the classic survey \cite{ljung1998system}, and the extensive overview of recent literature in \cite{hardt2016gradient}. 

Ghahramani and Roweis \cite{roweis1999unifying} suggest using the EM algorithm to learn the parameters of an LDS. This approach, which directly tackles the non-convex problem, is widely used in practice \cite{icml2010_Martens10a}. However, it remains a long-standing challenge to characterize the theoretical guarantees afforded by EM. We find that it is easy to produce cases where EM fails to identify the correct system.

In a recent result of \cite{hardt2016gradient}, it is shown for the first time that for a restricted class of systems, gradient descent (also widely used in practice, perhaps better known in this setting as backpropagation) guarantees polynomial convergence rates and sample complexity in the batch setting. Their result applies essentially only to the SISO case (vs. multi-dimensional for us), depends polynomially on the spectral gap (as opposed to no dependence for us), and requires the signal to be created by an LDS (vs. arbitrary for us). 



\section{Preliminaries}

\subsection{Linear dynamical systems}
Many different settings have been considered, in which the definition of an LDS takes on many variants. We are interested in discrete time-invariant MIMO (multiple input, multiple output) systems with a finite-dimensional hidden state.\footnote{We assume finite dimension for simplicity of presentation. However, it will be evident that hidden-state dimension has no role in our algorithm, and shows up as $\norm{B}_F$ and $\norm{C}_F$ in the regret bound.} Formally, our model is given as follows:
\begin{definition}
A linear dynamical system (LDS) is a map from a sequence of input vectors $x_1, \ldots, x_T \in \R^n$ to output (response) vectors $y_1, \ldots, y_T \in \R^m$ of the form
\begin{align}
h_{t+1} &= A h_t + B x_t + \eta_t \\
y_t &= C h_t + D x_t + \xi_t,
\end{align}
where $h_0, \ldots, h_T \in \R^d$ is a sequence of hidden states, $A,B,C,D$ are matrices of appropriate dimension, and $\eta_t \in \R^d, \xi_t \in \R^m$ are (possibly stochastic) noise vectors.
\end{definition}

Unrolling this recursive definition gives the \emph{impulse response function}, which uniquely determines the LDS. For notational convenience, for invalid indices $t \leq 0$, we define $x_t$, $\eta_t$, and $\xi_t$ to be the zero vector of appropriate dimension. Then, we have:
\begin{equation}
\label{eqn:lds-impulse-response}
y_t = \sum_{i=1}^{T-1} C A^i \pa{ B x_{t-i} + \eta_{t-i} } + CA^t h_0 + D x_t + \xi_t.
\end{equation}
We will consider the (discrete) time derivative of the impulse response function, given by expanding $y_{t-1} - y_t$ by Equation~(\ref{eqn:lds-impulse-response}).
For the rest of this paper, we focus our attention on systems subject to the following restrictions:
\begin{enumerate}
\item[(i)] The LDS is \emph{Lyapunov stable}: $\norm{A}_2 \leq 1$, where $\norm{\cdot}_2$ denotes the operator (a.k.a. spectral) norm.
\item[(ii)] The transition matrix $A$ is symmetric and positive semidefinite.\footnote{The psd constraint on $A$ can be removed by augmenting the inputs $x_t$ with extra coordinates $(-1)^t (x_t)$. We omit this for simplicity of presentation. }
\end{enumerate}

The first assumption is standard: when the hidden state is allowed to blow up exponentially, fine-grained prediction is futile. In fact, many algorithms only work when $\norm{A}$ is \emph{bounded away} from $1$, so that the effect of any particular $x_t$ on the hidden state (and thus the output) dissipates exponentially. We do not require this stronger assumption.

We take a moment to justify assumption (ii), and why this class of systems is still expressive and useful. First, symmetric LDSs constitute a natural class of linearly-observable, linearly-controllable systems with dissipating hidden states (for example, physical systems with friction or heat diffusion). Second, this constraint has been used successfully for video classification and tactile recognition tasks \cite{huang2016sparse}. Interestingly, though our theorems require symmetric $A$, our algorithms appear to tolerate some non-symmetric (and even nonlinear) transitions in practice.

\subsection{Sequence prediction as online regret minimization}
\label{subsection:prelim-oco}
A natural formulation of system identification is that of \emph{online sequence prediction}. At each time step $t$, an online learner is given an input $x_t$, and must return a predicted output $\hat y_t$. Then, the true response $y_t$ is observed, and the predictor suffers a squared-norm loss of $\norm{ y_t - \hat y_t }^2$. Over $T$ rounds, the goal is to predict as accurately as the best LDS in hindsight.

Note that the learner is permitted to access the history of observed responses $\{y_1, \ldots, y_{t-1}\}$. Even in the presence of statistical (non-adversarial) noise, the fixed maximum-likelihood sequence produced by $\Theta = (A,B,C,D,h_0)$ will accumulate error linearly as $T$. Thus, we measure performance against a more powerful comparator, which fixes LDS parameters $\Theta$, and predicts $y_t$ by the previous response $y_{t-1}$ plus the derivative of the impulse response function of $\Theta$ at time $t$.

We will exhibit an online algorithm that can compete against the best $\Theta$ in this setting.
Let $\hat y_1, \ldots, \hat y_T$ be the predictions made by an online learner, and let $y_1^*, \ldots, y_T^*$ be the sequence of predictions, realized by a chosen setting of LDS parameters $\Theta$, which minimize total squared error. Then, we define regret by the difference of total squared-error losses:
\[ \Regret(T) \defeq \sum_{t=1}^T \norm{ y_t - \hat y_t }^2 - \sum_{t=1}^T \norm{ y_t - y^*_t }^2. \]
This setup fits into the standard setting of online convex optimization (in which a sublinear regret bound implies convergence towards optimal predictions), save for the fact that the loss functions are non-convex in the system parameters. Also, note that a randomized construction (set all $x_t = 0$, and let $y_t$ be i.i.d. Bernoulli random variables) yields a lower bound\footnote{This is a standard construction; see, e.g. Theorem~3.2 in \cite{OCObook}.} for any online algorithm: $\E\bra{ \Regret(T) } \geq \Omega ( \sqrt{T} )$.

To quantify regret bounds, we must state our scaling assumptions on the (otherwise adversarial) input and output sequences. We assume that the inputs are bounded: $\norm{x_t}_2 \leq R_x$. Also, we assume that the output signal is Lipschitz in time: $\norm{y_t - y_{t-1}}_2 \leq L_y$. The latter assumption exists to preclude pathological inputs where an online learner is forced to incur arbitrarily large regret. For a true noiseless LDS, $L_y$ is not too large; see Lemma~\ref{lem:true-lds-lipschitz} in the appendix.

We note that an optimal $\tilde O(\sqrt{T})$ regret bound can be trivially achieved in this setting by algorithms such as Hedge \cite{LitWar94}, using an exponential-sized discretization of all possible LDS parameters; this is the online equivalent of brute-force grid search. Strikingly, our algorithms achieve essentially the same regret bound, but run in polynomial time.

\subsection{The power of convex relaxations}
\label{subsection:convex-relaxations}
Much work in system identification, including the EM method, is concerned with explicitly finding the LDS parameters $\Theta = (A,B,C,D,h_0)$ which best explain the data. However, it is evident from Equation~\ref{eqn:lds-impulse-response} that the $C A^i B$ terms cause the least-squares (or any other) loss to be non-convex in $\Theta$. Many methods used in practice, including EM and subspace identification, heuristically estimate each hidden state $h_t$, after which estimating the parameters becomes a convex linear regression problem. However, this first step is far from guaranteed to work in theory or practice.

Instead, we follow the paradigm of improper learning: in order to predict sequences as accurately as the best possible LDS $\Theta^* \in \Hcal$, one need not predict strictly from an LDS. The central driver of our algorithms is the construction of a slightly larger hypothesis class $\hat\Hcal$, for which the best predictor $\hat \Theta^*$ is nearly as good as $\Theta^*$. Furthermore, we construct $\hat\Hcal$ so that the loss functions \emph{are} convex under this new parameterization. From this will follow our efficient online algorithm.

As a warmup example, consider the following overparameterization: pick some time window $\tau \ll T$, and let the predictions $\hat y_t$ be linear in the concatenation $[x_t, \ldots, x_{t-\tau}] \in \R^{\tau d}$. When $\norm{A}$ is bounded away from 1, this is a sound assumption.\footnote{This assumption is used in \emph{autoregressive models}; see Section~6 of \cite{hardt2016gradient} for a theoretical treatment.} However, in general, this approximation is doomed to either truncate longer-term input-output dependences (short $\tau$), or suffer from overfitting (long $\tau$). Our main theorem uses an overparameterization whose approximation factor $\eps$ is independent of $\norm{A}$, and whose sample complexity scales only as $\tilde O( \mathrm{polylog}(T, 1 / \eps) )$. 

\subsection{Low approximate rank of Hankel matrices}
\label{subsection:hankel-spectrum}
Our analysis relies crucially on the spectrum of a certain \emph{Hankel matrix}, a square matrix whose anti-diagonal stripes have equal entries (i.e. $H_{ij}$ is a function of $i+j$). An important example is the Hilbert matrix $H_{n,\theta}$, the $n$-by-$n$ matrix whose $(i,j)$-th entry is $\frac{1}{i+j+\theta}$. For example,
\[H_{3,-1} = \bmat{1&\sfrac{1}{2}&\sfrac{1}{3}\\\sfrac{1}{2}&\sfrac{1}{3}&\sfrac{1}{4}\\\sfrac{1}{3}&\sfrac{1}{4}&\sfrac{1}{5}}. \]

This and related matrices have been studied under various lenses for more than a century: see, e.g., \cite{hilbert1894beitrag,choi1983tricks}. A basic fact is that $H_{n,\theta}$ is a positive definite matrix for every $n \geq 1, \theta > -2$. The property we are most interested in is that the spectrum of a positive semidefinite Hankel matrix decays exponentially, a difficult result derived in \cite{beckermann2016singular} via Zolotarev rational approximations. We state these technical bounds in Appendix~\ref{appendix-section:hankel-properties}.


\section{The wave-filtering algorithm}

\label{section:alg}
Our online algorithm (Algorithm~\ref{alg:wave-ogd}) runs online projected gradient descent \cite{Zinkevich03} on the squared loss $f_t(M_t) \defeq \norm{y_t - \hat y_t(M_t)}^2$. Here, each $M_t$ is a matrix specifying a linear map from featurized inputs $\tilde X_t$ to predictions $\hat y_t$. Specifically, after choosing a certain bank of $k$ \emph{filters} $\{\phi_j\}$, $\tilde X_t \in \R^{nk + 2n + m}$ consists of convolutions of the input time series with each $\phi_j$ (scaled by certain constants), along with $x_{t-1}$, $x_t$, and $y_{t-1}$. The number of filters $k$ will turn out to be polylogarithmic in $T$.

The filters $\{\phi_j\}$ and scaling factors $\{\sigma_j^{1/4}\}$ are given by the top eigenvectors and eigenvalues of the Hankel matrix $Z_T \in \R^{T \times T}$, whose entries are given by
\[Z_{ij} := \frac{2}{(i+j)^3 - (i+j)}. \]

In the language of Section~\ref{subsection:convex-relaxations}, one should think of each $M_t$ as arising from an $\tilde O( \mathrm{poly}(m,n,d,\log T) )$-dimensional hypothesis class $\hat\Hcal$, which replaces the original $O((m+n+d)^2)$-dimensional class $\Hcal$ of LDS parameters $(A,B,C,D,h_0)$. Theorem~\ref{thm:main-appx-relaxation} gives the key fact that $\hat \Hcal$ approximately contains $\Hcal$.

\begin{algorithm}
\caption{Online wave-filtering algorithm for LDS sequence prediction}
\label{alg:wave-ogd}
\begin{algorithmic}[1]
\STATE Input: time horizon $T$, filter parameter $k$, learning rate $\eta$, radius parameter $R_M$.
\STATE Compute $\{(\sigma_j, \phi_j)\}_{j=1}^k$, the top $k$ eigenpairs of $Z_T$.
\STATE Initialize $M_1 \in \R^{m \times k'}$, where $k' \defeq nk + 2n + m$.
\FOR{$t = 1, \ldots, T$}
\STATE Compute $\tilde X \in \R^{k'}$, with first $nk$ entries $\tilde X_{(i,j)} := \sigma_j^{1/4} \sum_{u=1}^{T-1} \phi_j(u) x_{t-u}(i)$, followed by the $2n+m$ entries of $x_{t-1}$, $x_t$, and $y_{t-1}$.
\STATE Predict $\hat y_t := M_t \tilde X$.
\STATE Observe $y_t$. Suffer loss $\norm{y_t - \hat y_t}^2$.
\STATE Gradient update: $M_{t+1} \leftarrow M_t - 2 \eta (y_t - \hat y_t) \otimes \tilde X$.
\IF{$\norm{ M_{t+1} }_F \geq R_M$}
\STATE Perform Frobenius norm projection: $M_{t+1} \leftarrow \frac{R_M}{\norm{ M_{t+1} }_F} M_{t+1}$.
\ENDIF
\ENDFOR
\end{algorithmic}
\end{algorithm}

In Section~\ref{section:analysis}, we provide the precise statement and proof of Theorem~\ref{thm:main-online}, the main regret bound for Algorithm~\ref{alg:wave-ogd}, with some technical details deferred to the appendix.
We also obtain analogous sample complexity results for batch learning;
however, on account of some definitional subtleties,
we defer all discussion of the offline case, including the statement and proof of Theorem~\ref{thm:main-batch}, to Appendix~\ref{appendix-section:batch}.

We make one final interesting note here, from which the name \emph{wave-filtering} arises: when plotted coordinate-wise, our filters $\{\phi_j\}$ look like the vibrational modes of an inhomogeneous spring (see Figure~\ref{fig:filters-demo}). We provide some insight on this phenomenon (along with some other implementation concerns) in Appendix~\ref{appendix-section:extensions}. Succinctly: in the scaling limit, $(Z_T / \norm{Z_T}_2)_{T \rightarrow \infty}$ commutes with a certain second-order Sturm-Liouville differential operator $\Dcal$. This allows us to approximate filters with eigenfunctions of $\Dcal$, using efficient numerical ODE solvers.

\begin{figure}
  \centering
  \begin{subfigure}[b]{0.33\linewidth}
  \includegraphics[width=\linewidth,trim={10mm 5mm 5mm 5mm},clip]{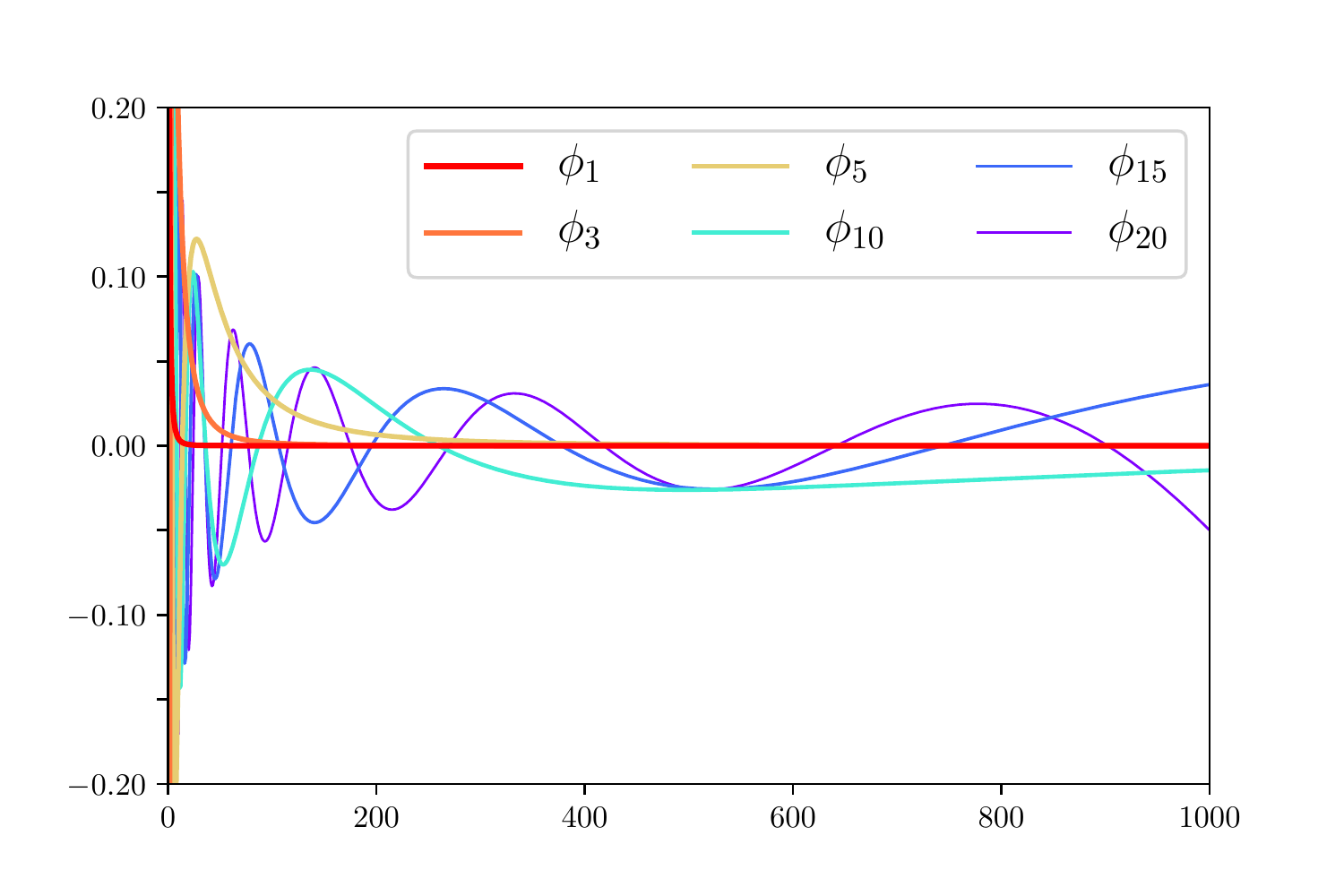}
  \caption{}
  \end{subfigure}%
  \begin{subfigure}[b]{0.33\linewidth}
  \includegraphics[width=\linewidth,trim={10mm 5mm 5mm 5mm},clip]{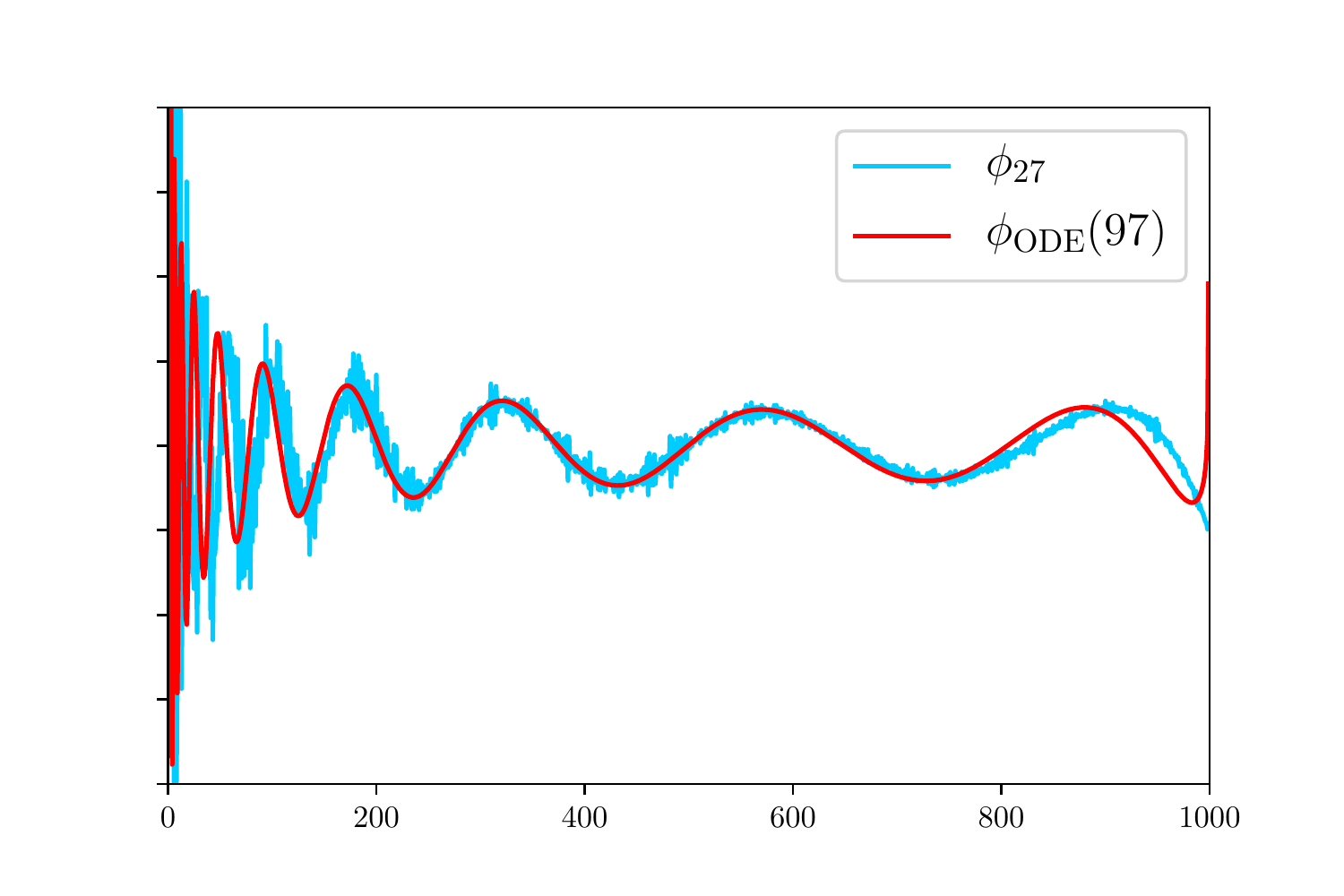}
  \caption{}
  \end{subfigure}%
  \begin{subfigure}[b]{0.33\linewidth}
  \includegraphics[width=\linewidth,trim={10mm 5mm 5mm 5mm},clip]{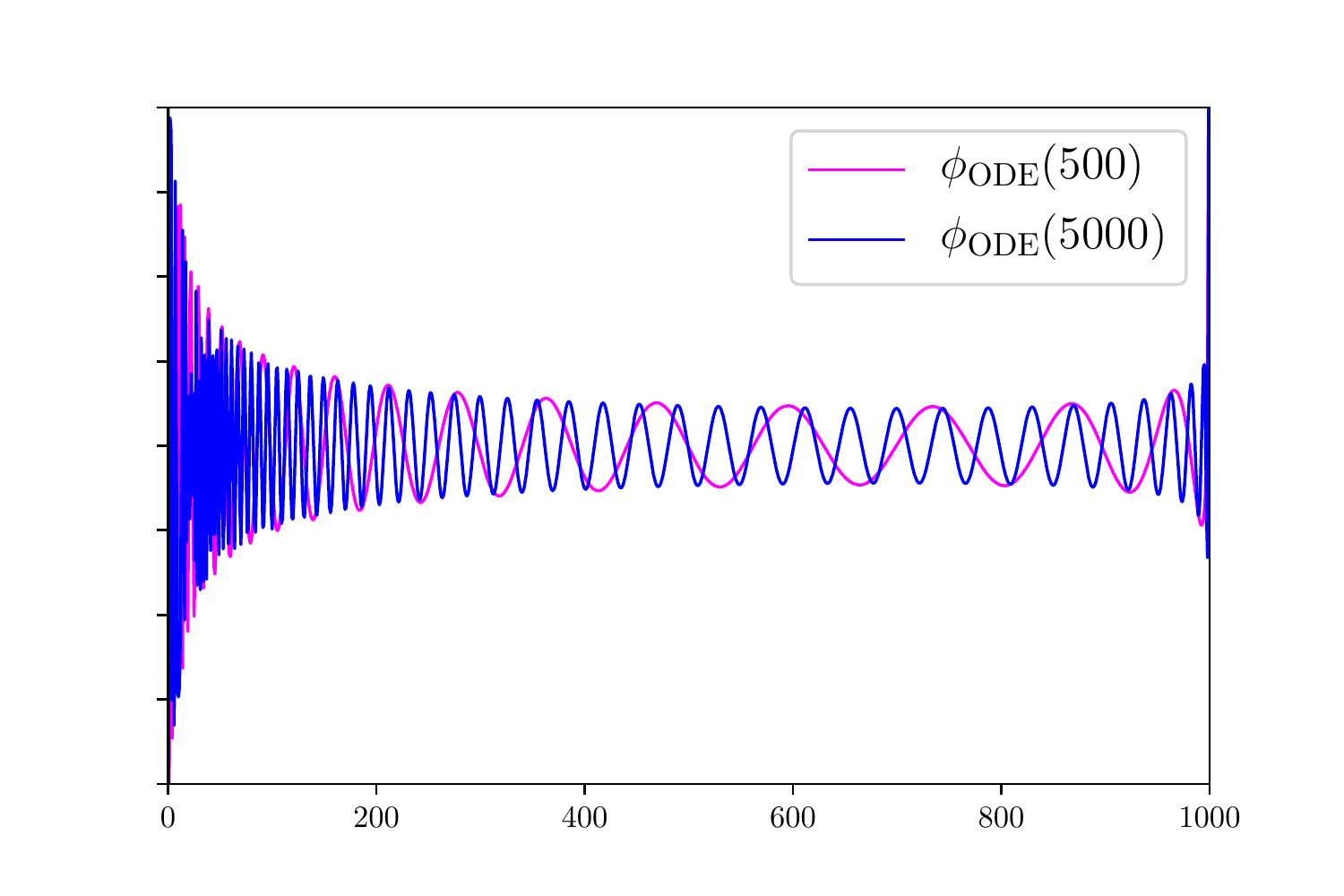}
  \caption{}
  \end{subfigure}

  \caption{(a) The entries of some typical eigenvectors of $Z_{1000}$, plotted coordinate-wise. (b) $\phi_{27}$ of $Z_{1000}$ ($\sigma_{27} \approx 10^{-16}$) computed with finite-precision arithmetic, along with a numerical solution to the ODE in Appendix~\ref{appendix-section:extensions}.1 with $\lambda = 97$. (c) Some very high-order filters, computed using the ODE, would be difficult to obtain by eigenvector computations. }
  \label{fig:filters-demo}
\end{figure}


\section{Analysis}

\label{section:analysis}

We first state the full form of the regret bound achieved by Algorithm~\ref{alg:wave-ogd}:\footnote{Actually, for a slightly tighter proof, we analyze a restriction of the algorithm which does not learn the portion $M^{(y)}$, instead always choosing the identity matrix for that block.}

\begin{customthm}{\ref{thm:main-online}}[Main]
On any sequence $\{ (x_t, y_t) \}_{t=1}^T$, Algorithm~\ref{alg:wave-ogd}, with a choice of $k = \Theta\pa{ \log^2 T \log (R_\Theta R_x L_y n) }$, $R_M = \Theta(R_\Theta^2 \sqrt{k})$,
and $\eta = \Theta( (R_x^2 L_y \, \log (R_\Theta R_x L_y n) \, n \sqrt{T} \log^4 T)^{-1} )$, achieves regret
\[ \Regret(T) \leq O\pa{ R_\Theta^4 \, R_x^2 \, L_y \, \log^2 (R_\Theta R_x L_y n) \cdot n \sqrt{T} \log^6 T }, \]
competing with LDS predictors $(A,B,C,D,h_0)$ with $0 \preccurlyeq A \preccurlyeq I$ and $\norm{B}_F, \norm{C}_F, \norm{D}_F, \norm{h_0} \leq R_\Theta$.
\end{customthm}

Note that the dimensions $m,d$ do not appear explicitly in this bound, though they typically factor into $R_\Theta$.
In Section~\ref{subsection:analysis-convex-relaxation}, we state and prove Theorem~\ref{thm:main-appx-relaxation}, the convex relaxation guarantee for the filters, which may be of independent interest. This allows us to approximate the optimal LDS in hindsight (the regret comparator) by the loss-minimizing matrix $M_t: \tilde X \mapsto \hat y_t$. In Section~\ref{subsection:analysis-low-regret}, we complete the regret analysis using Theorem~\ref{thm:main-appx-relaxation}, along with bounds on the diameter and gradient, to conclude Theorem~\ref{thm:main-online}.

Since the batch analogue is less general (and uses the same ideas), we defer discussion of Algorithm~\ref{alg:wave-batch} and Theorem~\ref{thm:main-batch} to Appendix~\ref{appendix-section:batch}.

\subsection{Approximate convex relaxation via wave filters}
\label{subsection:analysis-convex-relaxation}
Assume for now that $h_0 = 0$; we will remove this at the end, and see that the
regret bound is asymptotically the same.
Recall (from Section~\ref{subsection:prelim-oco})
that we measure regret compared to predictions obtained by adding the derivative of the impulse response function of an LDS $\Theta$ to $y_{t-1}$.
Our approximation theorem states that for any $\Theta$, there is some $M_{\Theta} \in \hat\Hcal$ which produces approximately the same predictions.
Formally:
\begin{customthm}{3}[Spectral convex relaxation for symmetric LDSs]
\label{thm:main-appx-relaxation}
Let $\{\hat y_t\}_{t=1}^T$ be the online predictions made by an LDS $\Theta = (A,B,C,D,h_0 = 0)$. Let $R_\Theta = \max \{ \norm{B}_F, \norm{C}_F, \norm{D}_F  \}$.
Then, for any $\eps > 0$, with a choice of $k = \Omega \pa{\log T \log (R_\Theta R_x L_y n T / \eps)}$, there exists an $M_\Theta \in \R^{m \times k'}$ such that
\[ \sum_{t=1}^T \norm{M_{\Theta} \tilde X_t - y_t}^2 \leq \sum_{t=1}^T \norm{\hat y_t - y_t}^2 + \eps. \]
Here, $k'$ and $\tilde X_t$ are defined as in Algorithm~\ref{alg:wave-ogd} (noting that $\tilde X_t$ includes the previous ground truth $y_{t-1}$).
\end{customthm}
\begin{proof}
We construct this mapping $\Theta \mapsto M_\Theta$ explicitly. Write $M_\Theta$ as the block matrix
\[\bmat{ M^{(1)} & M^{(2)} & \cdots & M^{(k)} & M^{(x')} & M^{(x)} & M^{(y)} },\]
where the blocks' dimensions are chosen to align with $\tilde X_t$, the concatenated vector
\[\bmat{ \sigma_1^{1/4}(X * \phi_1)_t & \sigma_2^{1/4}(X * \phi_2)_t & \cdots & \sigma_k^{1/4}(X * \phi_k)_t & \; x_{t-1} \; & \; x_t \; & \; y_{t-1} },\]
so that the prediction is the block matrix-vector product
\[ M_{\Theta} \tilde X_t = \sum_{j=1}^k \sigma_j^{1/4} M^{(j)} (X * \phi_j)_t + M^{(x')} x_{t-1} + M^{(x)} x_t + M^{(y)} y_{t-1}. \]
Without loss of generality, assume that $A$ is diagonal, with entries $\{\alpha_l\}_{l=1}^d$.\footnote{Write the eigendecomposition $A = U \Lambda U^T$. Then, the LDS with parameters $(\hat A, \hat B, \hat C, D, h_0) := (\Lambda, BU, U^T C, D, h_0)$ makes the same predictions as the original, with $\hat A$ diagonal.} Let $b_l$ be the $l$-th row of $B$, and $c_l$ the $l$-th column of $C$. Also, we define a continuous family of vectors $\mu: [0,1] \rightarrow \R^T$, with entries $\mu(\alpha)(i) = (\alpha_l - 1) \alpha_l^{i-1}$.
Then, our construction is as follows:
\begin{itemize}
\item $M^{(j)} = \sum_{l=1}^d \sigma_j^{-1/4} \ang{\phi_j,\mu(\alpha_l)} (c_l \otimes b_l)$, for each $1 \leq j \leq k$.
\item $M^{(x')} = -D,\quad M^{(x)} = CB + D,\quad M^{(y)} = I_{m\times m}$.
\end{itemize}
Below, we give the main ideas for why this $M_\Theta$ works, leaving the full proof to Appendix~\ref{appendix-section:approximate-relaxation}.

Since $M^{(y)}$ is the identity, the online learner's task is to predict the differences $y_t - y_{t-1}$ as well as the derivative $\Theta$, which we write here:
\begin{align}
\hat y_t - y_{t-1} &= (CB + D)x_t - Dx_{t-1} + \sum_{i=1}^{T-1} C(A^i - A^{i-1})Bx_{t-i} \nonumber \\
&= (CB + D)x_t - Dx_{t-1} + \sum_{i=1}^{T-1} C \pa{\sum_{l=1}^d \pa{ \alpha_l^i - \alpha_l^{i-1} } e_l \otimes e_l} Bx_{t-i} \nonumber \\
&= (CB + D)x_t - Dx_{t-1} + \sum_{l=1}^d (c_l \otimes b_l) \sum_{i=1}^{T-1} \mu(\alpha_l)(i) \, x_{t-i}. \label{eq:exact-m}
\end{align}
Notice that the inner sum is an inner product between each coordinate of the past inputs $(x_t, x_{t-1}, \ldots, x_{t-T})$ with $\mu(\alpha_l)$ (or a convolution, viewed across the entire time horizon). The crux of our proof is that one can approximate $\mu(\alpha)$ using a linear combination of the filters $\{\phi_j\}_{j=1}^k$. Writing $Z := Z_T$ for short, notice that
\[Z = \int_0^1 \mu(\alpha) \otimes \mu(\alpha) \, d\alpha, \]
since the $(i,j)$ entry of the RHS is
\[\int_0^1 (\alpha - 1)^2 \alpha^{i+j-2} \,d\alpha = \frac{1}{i+j-1} - \frac{2}{i+j} + \frac{1}{i+j+1} = Z_{ij}.\]
What follows is a spectral bound for reconstruction error, relying on the low approximate rank of $Z$:
\begin{lemma}
\label{lem:pca-lemma} Choose any $\alpha \in [0, 1]$. Let $\tilde\mu(\alpha)$ be the projection of $\mu(\alpha)$ onto the $k$-dimensional subspace of $\R^T$ spanned by $\{\phi_j\}_{j=1}^k$. Then,
\[\norm{\mu(\alpha) - \tilde\mu(\alpha)}^2 \leq \sqrt{ 6 \sum_{j=k+1}^T \sigma_{j} } \leq O\pa{ c_0^{-k/\log T} \sqrt{ \log T } },\]
for an absolute constant $c_0 > 3.4$.
\end{lemma}
By construction of $M^{(j)}$, $M_\Theta \tilde X_t$ replaces each $\mu(\alpha_l)$ in Equation~(\ref{eq:exact-m}) with its approximation $\tilde\mu(\alpha_l)$. Hence we conclude that
\begin{align*}
M_\Theta \tilde X_t &= y_{t-1} + (CB + D)x_t - Dx_{t-1} + \sum_{l=1}^d (c_l \otimes b_l) \sum_{i=1}^{T-1} \tilde\mu(\alpha_l)(i) \, x_{t-i}\\
&= y_{t-1} + (\hat y_t - y_{t-1}) + \zeta_t \quad = \; \hat y_t + \zeta_t,
\end{align*}
letting $\{\zeta_t\}$ denote some residual vectors arising from discarding the subspace of dimension $T-k$.
Theorem~\ref{thm:main-appx-relaxation} follows by showing that these residuals are small, using Lemma~\ref{lem:pca-lemma}: it turns out that $\norm{\zeta_t}$ is exponentially small in $k/\log T$, which implies the theorem.
\end{proof}

\subsection{From approximate relaxation to low regret}
\label{subsection:analysis-low-regret}
Let $\Theta^* \in \Hcal$ denote the best LDS predictor, and let $M_{\Theta^*} \in \hat \Hcal$
be its image under the map from Theorem~\ref{thm:main-appx-relaxation}, so that total squared error of predictions $M_{\Theta^*} \tilde X_t$ is within $\eps$ from that of $\Theta^*$. Notice that the loss functions $f_t(M) \defeq \norm{y_t - M \tilde X_t}^2$ are quadratic in $M$, and thus convex. Algorithm~\ref{alg:wave-ogd} runs online gradient descent~\cite{Zinkevich03} on these loss functions, with decision set $\Mcal \defeq \{ M \in \R^{m \times k'} \;\big|\; \norm{M}_F \leq R_M \}$. Let $D_\mathrm{max} := \sup_{M,M' \in \Mcal} \norm{M - M'}_F$ be the diameter of $\Mcal$, and $G_\mathrm{max} := \sup_{M \in \Mcal, \tilde X} \norm{\nabla f_t(M)}_F$ be the largest norm of a gradient. We can invoke the classic regret bound:
\begin{lemma}[e.g. Thm.~3.1 in \cite{OCObook}]
\label{thm:ogd-textbook-bound}
Online gradient descent, using learning rate $\frac{D_\mathrm{max}}{G_\mathrm{max}\sqrt{T}}$, has regret
\[ \Regret_\mathrm{OGD}(T) \defeq \sum_{t=1}^T f_t (M_t) - \min_{M \in \Mcal} \sum_{t=1}^T f_t (M) \; \leq \; 2G_\mathrm{max} D_\mathrm{max}\sqrt{T}. \]
\end{lemma}
To finish, it remains to show that $D_\mathrm{max}$ and $G_\mathrm{max}$ are small. In particular, since the gradients contain convolutions of the input by $\ell_2$ (not $\ell_1$) unit vectors, special care must be taken to ensure that these do not grow too quickly. These bounds are shown in Section~\ref{appendix-subsection:diameter-gradient}, giving the correct
regret of Algorithm~\ref{alg:wave-ogd} in comparison with the comparator $M^* \in \hat \Hcal$. By Theorem~\ref{thm:main-appx-relaxation}, $M^*$ competes arbitrarily closely with the best LDS in hindsight, concluding the theorem.

Finally, we discuss why it is possible to relax the earlier assumption $h_0 = 0$ on the initial hidden state. Intuitively, as more of the ground truth responses $\{y_t\}$ are revealed, the largest possible effect of the initial state decays. Concretely, in Section~\ref{appendix-subsection:hidden-state}, we prove that a comparator who chooses a nonzero $h_0$
can only increase the regret by an additive $\tilde O(\log^2 T)$ in the online setting.


\section{Experiments}
In this section, to highlight the appeal of our provable method, we exhibit two minimalistic cases where
traditional methods for system identification fail, while ours successfully learns the system.
Finally, we note empirically that our method seems not to degrade in practice on certain well-behaved nonlinear systems.
In each case, we use $k = 25$ filters, and a regularized follow-the-leader variant of Algorithm~\ref{alg:wave-ogd} (see Appendix~\ref{appendix-section:extensions}.2).

\begin{figure}
  \centering
  \begin{subfigure}{\textwidth}\centering
  \includegraphics[width=\textwidth]{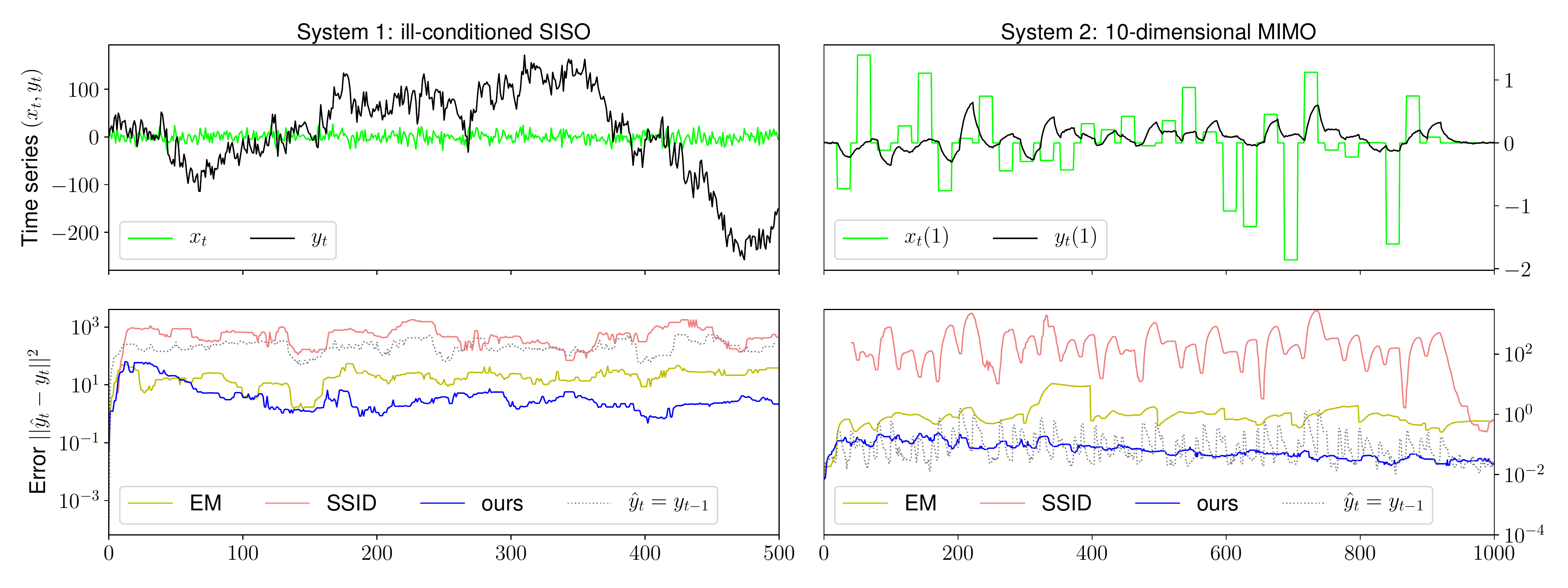}
  \caption{Two synthetic systems. For clarity, error plots are smoothed by a median filter. \emph{Left:} Noisy SISO system with a high condition number; EM and SSID finds a bad local optimum.
  \emph{Right:} High-dimensional MIMO system; other methods fail to learn any reasonable model of the dynamics.}
  \end{subfigure}

  \begin{subfigure}{\textwidth}\centering
  \includegraphics[width=\textwidth]{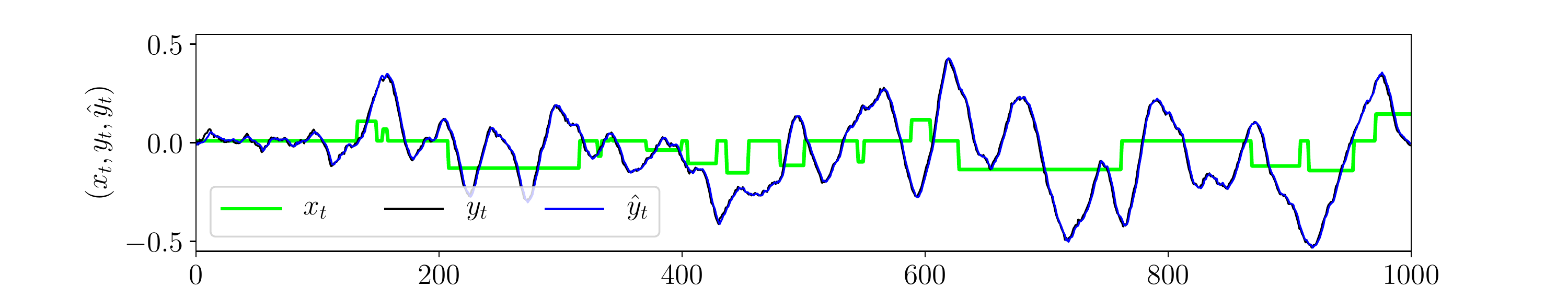}
  \caption{Forced pendulum, a physical simulation our method learns in practice, despite a lack of theory.}
  \end{subfigure}

  \caption{\small{ Visualizations of Algorithm~\ref{alg:wave-ogd}. All plots: \textcolor{blue}{blue} = ours, \textcolor[rgb]{0.7,0.7,0}{yellow} = EM, \textcolor{red}{red} = SSID, \textbf{black} = true responses, \textcolor{green}{green} = inputs, dotted lines = ``guess the previous output'' baseline. Horizontal axis is time. }}
\end{figure}

\subsection{Synthetic systems: two hard cases for EM and SSID}
We construct two difficult systems, on which we run either EM or subspace identification\footnote{Specifically, we use ``Deterministic Algorithm 1'' from page 52 of \cite{van2012subspace}.} (SSID), followed by Kalman filtering to obtain predictions.
Note that our method runs significantly ($>\kern-0.3em 1000$ times) faster than this traditional pipeline.

In the first example (Figure 2(a), left), we have a SISO system ($n=m=1$) and $d=2$; all $x_t$, $\xi_t$, and $\eta_t$ are i.i.d. Gaussians, and $B^\top = C = [1\;\;1], D = 0$. Most importantly, $A = \mathsf{diag}\pa{ [0.999,0.5] }$ is ill-conditioned, so that there are long-term dependences between input and output. Observe that although EM and SSID both find reasonable guesses for the system's dynamics, they turns out to be local optima. Our method learns to predict as well as the \emph{best possible} LDS.

The second example (Figure 2(a), right) is a MIMO system (with $n=m=d=10$), also with Gaussian noise. The transition matrix $A = \mathsf{diag}\pa{ [0, 0.1, 0.2, \ldots, 0.9] }$ has a diverse spectrum, the observation matrix $C$ has i.i.d. Gaussian entries, and $B = I_n, D = 0$. The inputs $x_t$ are random block impulses. This system identification problem is high-dimensional and non-convex; it is thus no surprise that EM and SSID  consistently fail to converge.

\subsection{The forced pendulum: a nonlinear, non-symmetric system}
We remark that although our algorithm has provable regret guarantees only for LDSs with symmetric transition matrices, it appears in experiments to succeed in learning some non-symmetric (even nonlinear) systems in practice, much like the unscented Kalman filter~\cite{wan2000unscented}. In Figure~2(b), we provide a typical learning trajectory for a forced pendulum, under Gaussian noise and random block impulses. Physical systems like this are widely considered in control and robotics, suggesting possible real-world applicability for our method.


\section{Conclusion}
We have proposed a novel approach for provably and efficiently learning linear dynamical systems. Our online \emph{wave-filtering} algorithm attains near-optimal regret in theory; and experimentally outperforms traditional system identification in both prediction quality and running time. Furthermore, we have introduced a ``spectral filtering'' technique for convex relaxation, which uses convolutions by eigenvectors of a Hankel matrix. We hope that this theoretical tool will be useful in tackling more general cases, as well as other non-convex learning problems.

\section*{Acknowledgments}
We thank Holden Lee and Yi Zhang for helpful discussions.
We especially grateful to Holden for a thorough reading of our manuscript, and for pointing out a way to tighten the result in Lemma~C.1.

\nocite{*}
\bibliography{main}
\bibliographystyle{alpha}

\newpage

\appendix

\section*{Guide to the Appendix}

\begin{itemize}
\item In Appendix~\ref{appendix-section:batch}, we present two formulations of the batch learning equivalent of the online algorithm, and derive Theorem~\ref{thm:main-batch}, a companion sample complexity bound.
\item In Appendix~\ref{appendix-section:extensions}, we discuss some variants of our online algorithm, and offer some tips for implementation. We also provide discussion on the connection
of our filters to eigenfunctions of a certain differential operator.
\item In Appendix~\ref{appendix-section:approximate-relaxation}, we prove the key approximate convex relaxation result (Theorem~\ref{thm:main-appx-relaxation}).
\item In Appendix~\ref{appendix-section:analysis-low-regret}, we complete the details for the proof sketch provided in Section~\ref{subsection:analysis-low-regret}, concluding the main theorem, the regret bound for the online algorithm. Importantly, we address the subtle issue of deriving upper bounds for the gradient and diameter of the decision set.
\item In Appendix~\ref{appendix-section:hankel-properties}, we derive explicit non-asymptotic bounds for quantities of interest pertaining to the Hankel matrix $Z$, notably spectral decay. Key results are adapted from \cite{beckermann2016singular}.
\item In Appendix~\ref{appendix-section:moment-curve}, we verify some easy-to-prove properties of the important vector $\mu(\alpha)$,
for sake of completeness.
\end{itemize}

\section{Batch variants of the algorithm}

\label{appendix-section:batch}

The online prediction setting is sensitive to permutation of the time series:
that is, the same LDS does not in general map 
$\{ x_{\sigma(1)}, \ldots, x_{\sigma(T)} \}$
to
$\{ y_{\sigma(1)}, \ldots, y_{\sigma(T)} \}$.
As such, one must take care when defining the batch case:
the output time series (and thus, loss functions) are correlated,
so it is not meaningful to assume that they are i.i.d. samples from a distribution.
Thus, our online regret bound, which concerns a single episode, does not translate directly. However, our convex relaxation technique still allows us to do efficient improper learning
with least-squares regression, giving interesting and novel statistical guarantees.
In this section, we provide two possible formulations of the batch setting,
along with accompanying theorems.

In both cases, it is most natural to fix an \emph{episode length} $T$,
and consider a rollout of the system $\{ (x_t, y_t) \}_{t=1}^T$ to be a single example.
For short, let $X_i \in \R^{Tn}$ denote the concatenated vector of inputs for a single example,
and $Y_i \in \R^{Tm}$ the concatenated responses.
The batch formulation is to learn the dynamics of the system using $N$ samples $\{(X_i, Y_i)\}$.
Recall that the samples satisfy $\norm{x_t}_2 \leq R_x$ and $\norm{y_{t}-y_{t-1}}_2 \leq L_y$.

Unlike in the online setting, it will be less confusing in the batch setting to measure
the \emph{mean} squared error of predictions, rather than the total squared error.
Thus, in this section, $\ell_{X,Y}(h)$ will always refer to mean squared error.
As well, to follow statistical learning conventions (for ease of reading), we use $h$ to denote
a hypothesis (an LDS) instead of $\Theta$; this is distinguished from the hidden state $h_t$.

\subsection{Learning the derivative: the direct analogue}

Throughout this subsection, assume that $h_0 = 0$.

As noted, the sequential prediction algorithm can be restricted so as to never
make updates to the submatrix $M^{(y)}$, keeping it to be the identity matrix.
Notice that all other features in $\tilde X$ consist of inputs $x_t$ and their convolutions.
In other words, we can take the view that the matrix $M_t$
can be used to predict the \emph{differences} $y_t - y_{t-1}$ between successive
responses, as a function of the entire (aligned) input time series $(x_t, x_{t-1}, \ldots, x_{t_T})$.

Thus, we can formulate a direct analogue for the online algorithm: learn the mapping from an input time series $X_i \in \R^{Tn}$ to the \emph{differences} $Y_i' \in \R^{Tm}$, the concatenation of all $y_t - y_{t-1}$. For this, we can use Theorem~\ref{thm:main-appx-relaxation} (the approximation result) directly, and obtain an improper agnostic learning guarantee.

Specifically, let $\Hcal$ be a subset of the hypothesis class of LDS parameters $\Theta = (A,B,C,D,h_0 = 0)$, subject to $\norm{B}_F, \norm{C}_F, \norm{D}_F \leq R_\Theta$, and choose any approximation tolerance $\eps > 0$.\footnote{The distinction between measuring total vs. mean squared error is hidden in the constant in front of the $\log T$
when choosing the number of filters $k$.}
Then, Theorem~\ref{thm:main-appx-relaxation} states that choosing $\hat \Hcal$ with $k = \Omega\pa{ \log T \log (R_\Theta R_x L_y n T / \eps) }$ ensures the $\eps$-approximate relaxation property. In the language of the batch setting: for each $h \in \Hcal$ which predicts on the sample $(X, Y')$ with a mean squared error $\ell_X(h)$, there is some $\hat h \in \hat\Hcal$ so that
\[ \ell_{X,Y} (h) \leq \ell_{X,Y} (\hat h) + \eps. \]

The choice of batch algorithm is clear, in order to mimic Algorithm~\ref{thm:main-online}:
run least-squares regression on $\tilde X$ and $Y$, where $\tilde X$ is the same featurization
of the inputs as used in the online algorithm. We describe this procedure fully in Algorithm~\ref{alg:wave-batch}.

\setcounter{algorithm}{1}
\begin{algorithm}
\caption{Offline wave-filtering algorithm for learning the derivative of an LDS}
\label{alg:wave-batch}
\begin{algorithmic}[1]
\STATE Input: $S = \{ (X_i,Y'_i) \}$, a set of $N$ training samples, each of length $T$; filter parameter $k$.
\STATE Compute $\{(\sigma_j, \phi_j)\}_{j=1}^k$, the top $k$ eigenpairs of $Z_T$.
\STATE Initialize matrices $\mathbf{X} \in \R^{(nk+2n) \times NT}, \mathbf{Y'} \in \R^{m \times NT}$.
\FOR{each sample $(X, Y')$}
\FOR{$t = 1,\ldots,T$}
\STATE Compute $\tilde X_t \in \R^{nk+2n}$, with first $nk$ entries $\tilde X_{(i,j)} := \sigma_j^{1/4} \sum_{u=1}^{T-1} \phi_j(u) x_{t-u}(i)$, followed by the $2n$ entries of $x_{t-1}$, $x_t$.
\STATE Append $(\tilde X_t, Y'_t)$ as new columns to the matrices $\mathbf{X}, \mathbf{Y'}$.
\ENDFOR
\ENDFOR
\RETURN least-squares solution $(\mathbf{X} \mathbf{X}^\top)^{\dagger} \mathbf{X}^\top \mathbf{Y'}$.
\end{algorithmic}
\end{algorithm}

\subsubsection{Generalization bound}
By definition, Algorithm~\ref{alg:wave-batch} minimizes the empirical MSE loss
on the samples; as such, we can derive a PAC-learning bound for regression.
We begin with some definitions and assumptions, so that we can state the theorem.

As in the statement of the online algorithm, as a soft dimensionality restriction, we constrain the comparator class $\Hcal$ to contain LDSs with parameters $\Theta = (A,B,C,D,h_0 = 0)$ such that $0 \preccurlyeq A \preccurlyeq I$ and $\norm{B}_F, \norm{C}_F, \norm{D}_F, \norm{h_0} \leq R_\Theta$.
For an empirical sample set $S$, let $\ell_S(h) = \frac{1}{|S|} \sum_{(X,Y) \in S} \ell_{X,Y}(h)$. Similarly, for a distribution $\Dcal$, let $\ell_\Dcal(h) = \E_{(X,Y) \sim \Dcal} [\ell_{X,Y}(h)]$.

Then, we are able to obtain a sample complexity bound:
\begin{customthm}{\ref{thm:main-batch}}[Generalization of the batch algorithm]
Choose any $\eps > 0$.
Let $S = \{ (X_i, Y'_i) \}_{i=1}^N$ be a set of i.i.d. training samples from a distribution $\Dcal$. Let $\hat h \defeq \argmin_{h \in \hat\Hcal} \ell_S(h)$ be the output of Algorithm~\ref{alg:wave-batch}, with a choice of $k = \Theta( \log T \, \log(R_\Theta R_x L_y n T / \eps) )$.
Let $h^* \defeq \argmin_{h^* \in \Hcal} \ell_\Dcal(h)$ be the true loss minimizer. Then,
with probability at least $1 - \delta$, it holds that
\[ \ell_\Dcal( \hat h ) - \min_{h \in \Hcal} \ell_\Dcal( h )
\leq \eps + \frac{ O\pa{ R_\Theta^4 R_x^2 L_y \, \log^2(R_\Theta R_x L_y n / \eps) \, n \log^6 T + \sqrt{\log 1/\delta} } }{\sqrt{N}}.
\]
\end{customthm}
\begin{proof}
Lemma~\ref{lem:wheres-the-matrix} shows that we can restrict $\hat \Hcal$ by a Frobenius norm bound:
\[ \norm{ M }_F \leq O\pa{ R_\Theta^2 \sqrt{k} }. \]
Thus, the empirical Rademacher complexity of $\hat \Hcal$ on $N$ samples, with this restriction, thus satisfies
\[ \Rad_N(\hat \Hcal) \leq O\pa{ \frac{ R_\Theta^2 R_x \sqrt{k} }{\sqrt{N}} }. \]

Also, no single prediction error (and thus neither the empirical nor population loss)
will exceed the upper bound
\[\ell_\mathrm{max} \defeq \Theta(R_\Theta^4 R_x^2 L_y^2 k).\]

Finally, the loss is $G_\mathrm{max}$-Lipschitz in the matrix $h$,
where $G_\mathrm{max}$ is the same upper bound for the gradient as
mentioned in Section~\ref{subsection:analysis-low-regret}.
Lemma~\ref{lem:grad-small}, states that this is bounded by
$O \pa{ R_\Theta^2 R_x^2 L_y \cdot nk^{3/2} \log^2 T }$.

With all of these facts in hand, a standard Rademacher complexity-dependent generalization bound holds in the improper hypothesis class $\hat\Hcal$ (see, e.g. \cite{bartlett2002rademacher}):
\begin{lemma}[Generalization via Rademacher complexity]
\label{lem:rademacher}
With probability at least $1 - \delta$, it holds that
\begin{align*}
\ell_\Dcal( \hat h ) - \ell_\Dcal( \hat h^* ) &\leq
G_\mathrm{max} \Rad_N(\hat\Hcal)
+ \ell_\mathrm{max} \sqrt{ \frac{8 \ln 2/\delta}{N} }
\end{align*}
\end{lemma}
With the stated choice of $k$,
an upper bound for the RHS of Lemma~\ref{lem:rademacher} is
\[\frac{ O\pa{ R_\Theta^4 R_x^2 L_y \, \log^2(R_\Theta R_x L_y n / \eps) \, n \log^6 T + \sqrt{\log 1/\delta} } }{\sqrt{N}}.\]
Combining this with the approximation result (Theorem~\ref{thm:main-appx-relaxation})
yields the theorem.
\end{proof}

\subsection{The pure batch setting}
A natural question is whether there exists a batch learning algorithm that can use $X$ to predict $Y$ directly, as opposed to the differences $Y'$. This is possible in the regime of low noise: if one has predictions on $Y'$ that are correct up to MSE $\eps$, an easy solution is to integrate and obtain predictions for $Y$; however, the errors will accumulate to $T\eps$.
The same agnostic learning guarantee costs a rather dramatic factor of $T^2$ in sample complexity.

In the regime of low noise, an analogue of our approximation theorem (Theorem~\ref{thm:main-appx-relaxation}) is powerful enough to guarantee low error.
For convenience and concreteness, we record this here:
\begin{customthm}{3b}[Pure-batch approximation]
\label{thm:pure-appx}
Let $\Theta$ be an LDS specified by parameters $(A,B,C,D,h_0 = 0)$, with $0 \preccurlyeq A \preccurlyeq I$, and $\norm{B}_F, \norm{C}_F, \norm{D}_F \leq R_\Theta$. Suppose $\Theta$
takes an input sequence $X = \{x_1, \ldots, x_T\}$, and produces output sequence $Y = \{y_1, \ldots, y_T\}$, assuming all noise vectors $\xi_t, \eta_t$ are 0.
Then, for any $\eps > 0$, with a choice of $k = \Omega \pa{\log T \log (R_\Theta R_x L_y n T / \eps)}$, there exists an $M_\Theta \in \R^{m \times (nk + 2n)}$ such that
\[ \sum_{t=1}^T \bigg\lVert \pa{ \sum_{u=1}^t M_{\Theta} \tilde X_{u} } - y_t \bigg\rVert^2 \leq \sum_{t=1}^T \norm{\hat y_t - y_t}^2 + \eps, \]
where $\tilde X_t$ is defined as in Algorithm~\ref{alg:wave-ogd}, without the $y_{t-1}$ entries.
\end{customthm}

This fact follows from Theorem~\ref{thm:main-appx-relaxation}, setting $\eps / T$ as the desired precision; the cost of this additional precision is only a constant factor in $k$. Furthermore, this $M_\Theta$ is subject
to the same Frobenius norm constraint $\norm{M_\Theta}_F \leq O( R_\Theta^2 \sqrt{k} )$
as in Lemma~\ref{lem:wheres-the-matrix}.

\subsubsection{Filters from the Hilbert matrix}
Alternatively, in the realizable case (when the samples from $\Dcal$ are generated by an LDS, possibly with small noise), one can invoke a similar approximate relaxation theorem as Theorem~\ref{thm:main-appx-relaxation}. The filters become the eigenvectors of the Hilbert matrix $H_{T,-1}$,
the matrix whose $(i,j)$-th entry is $1/(i+j-1)$. This matrix exhibits the same spectral decay as $Z_T$; see~\cite{beckermann2016singular} for precise statements. the proof follows the sketch from Section~\ref{subsection:analysis-convex-relaxation}, approximating the powers of $\alpha_\ell$ by a spectral truncation of a different curve $\mu'(\alpha)(i) = \alpha^{i-1}$, sometimes called the \emph{moment curve} in $\R^T$. The Hilbert matrix arises from taking the second moment matrix of the uniform distribution on this curve.

However, we find that this approximation guarantee is insufficient to show the strong regret and agnostic learning bounds we exhibit for learning the derivative of the impulse response function. Nonetheless, we find that regression with these filters works well in practice,
even interchangeably in the online algorithm; see Section~\ref{appendix-subsection:ode}
for some intuition.

\subsection{Learning the initial hidden state via hints}
\label{appendix-subsection:batch-hidden-state}

In either of the above settings, it is not quite possible to apply the same argument as in
the online setting for pretending that the initial hidden state is zero.
When this assumption is removed, the quality of the convex relaxation
degrades by an \emph{additive} $\tilde O (\frac{ \log^2 T }{ T })$; see Section~\ref{appendix-subsection:hidden-state}. This does not matter much for the regret
bound, because it is subsumed by the worst-case regret of online gradient descent.

However, in the batch setting, we take the view of fixed $T$ and increasing $N$, so the contribution of the initial state is no longer asymptotically negligible. In other words, this additive approximation error hinders us from driving $\eps$ arbitrarily close to zero, no matter how many filters are selected.
In settings where $T$ is large enough, one may find this acceptable.

We present an augmented learning problem in which we $\emph{can}$ predict as
well as an LDS: the initial hidden state is provided in each sample, up to an arbitrary linear transformation. Thus, each sample takes the form $(X, Y, \tilde h_0)$, and it is
guaranteed that $\tilde h_0 = Qh_0$ for each sample, for a fixed matrix $Q \in \R^{d' \times d}$. This $Q$ must be well-conditioned for the problem to remain well-posed: our knowledge of $h_0$ should be in the same dynamic range as the ground truth. Concretely, we should assume that
$\sigma_\mathrm{max}(Q) / \sigma_\mathrm{min}(Q)$ is bounded.

The construction is as follows: append $d'$ ``dummy'' dimensions to the input, and add an impulse of $\tilde h_0$ in those dimensions at time 0. During the actual episode, these dummy inputs are always zero. Then, replacing $B$ with the augmented block matrix $[B \;\; Q^{-1}]$ recovers
the behavior of the system. Thus, we can handle this formulation of hidden-state learning in the online or batch setting, incurring no additional asymptotic factors.

\subsubsection{Initializations with finite support}

We highlight an important special case of the formulation discussed above, which is perhaps
the motivating rationale for this altered problem.

Consider a batch system identification setting in which there are only \emph{finitely many} initial states $h_0$ in the training and test data, and the experimenter can distinguish between these states. This can be interpreted a set of $n_\mathrm{hidden}$ known initial ``configurations'' of the system. Then, it is sufficient to augment the data with a one-hot vector in $\R^{n_\mathrm{hidden}}$, corresponding to the known initialization in each sample.
An important case is when $n_\mathrm{hidden} = 1$: when there is only \emph{one} distinct initial configuration; this occurs frequently in control problems.

In summary, the stated augmentation takes the original LDS with dimensions $(n,m,d,T)$,
and transforms it into one with dimensions $(n+n_\mathrm{hidden},m,d,T+1)$. The matrix $Q^{-1}$, as defined above, is the $n_\mathrm{hidden}$-by-$d$ matrix whose columns are the possible initial hidden states, which can be in arbitrary dimension. For convenience, we summarize this observation:

\begin{proposition}[Learning an LDS with few, distinguishable hidden states]
\label{thm:}
When there are $d'$ known hidden states, with $d' \norm{h_0} \leq R_\Theta$, Theorems 2, 3, and 3b apply to the modified LDS learning problem, with samples of the form $(\tilde h_0, X, Y)$. The dimension $n$ becomes $n + d'$.
\end{proposition}

\section{Implementation and variants}

\label{appendix-section:extensions}

We discuss the points mentioned in Section~\ref{section:alg} at greater length. Unlike the rest of the appendix, this section contains no technical proofs, and is intended as a user-friendly guide for making the wave-filtering method usable in practice.

\subsection{Computing the filters via Sturm-Liouville ODEs}
\label{appendix-subsection:ode}
We begin by expanding upon the observation, noted in Section~\ref{section:alg}, that the eigenvectors resemble inhomogeneously-oscillating waves, providing some justification for the heuristic numerical computation of the top eigenvectors of $Z_T$.

Computing the filters directly from $Z_T$ is difficult. In fact, the Hilbert matrix (its close cousin) is notoriously super-exponentially ill-conditioned; it is probably best known for being a pathological benchmark for finite-precision numerical linear algebra algorithms. One could ignore efficiency issues, and view this as a data-independent preprocessing step: these filters are deterministic. However, this difficult numerical problem poses an issue for using our method in practice.

Fortunately, as briefly noted in Section~\ref{section:alg}, some recourse is available. In \cite{grunbaum1982remark}, Gr\"unbaum constructs a tridiagonal matrix $T_{n,\theta}$ which commutes with each Hilbert matrix $H_{n,\theta}$, as defined in Section~\ref{subsection:hankel-spectrum}. In the appropriate scaling limit as $T \rightarrow \infty$, this $T_{n,\theta}$ becomes a Sturm-Liouville differential operator $\Dcal$ which does not depend on $\theta$, given by
\[\Dcal = \frac{d}{dx} \pa{ (1-x^2)x^2 \frac{d}{dx} } - 2x^2.\]

Notice that $Z_T = H_{T,-1} - 2H_{T,0} + H_{T,1}$. This suggests that for large $T$, the entries of the $\phi_j$ are approximated by solutions to the second-order ODE
\begin{equation}
\label{filter-ode}
\Dcal \phi = \lambda \phi.
\end{equation}
It is difficult to quantify theoretical bounds for this rather convoluted sequence of approximations; however, we find that this observation greatly aids with constructing these filters in practice. In particular, the map between eigenvalues $\sigma_j$ of $Z$ and $\lambda_j$ of $\Dcal$ corresponding to the same eigenvector/eigenfunction proves challenging to characterize for finite $T$. In practice, we find that our method's performance is sensitive to neither the precise eigenvalues nor the ODE boundary conditions.

In summary, aside from the name \emph{wave-filtering}, this observation yields a numerically stable recipe for computing filters (without a theorem): for each of $k$ hand-selected eigenvalues $\lambda$, compute a filter $\phi_\lambda$ using an efficient numerical solver to Equation~\ref{filter-ode}.

\subsection{Alternative low-regret algorithms}
We use online gradient descent as our prototypical low-regret learning algorithm due to its simplicity and stability under worst-case noise. However, in practice, particularly when there are additional structural assumptions on the data, we can replace the update step with that of any low-regret algorithm. AdaGrad \cite{adagrad} is a particularly appealing one, as it is likely to find learning rates which are better than those guaranteed theoretically.

Furthermore, if noise levels are relatively low, and it is known \emph{a priori} that the data are generated from a true LDS, a better approach might be to use follow-the-leader \cite{kalai2005efficient} or any of its variants. This amounts to replacing the update step with
\[M_{t+1} := \min_M \sum_{t'=1}^t \norm{y_{t'} - \hat y_{t'}(M)}^2,\]
a linear regression problem solvable via, e.g. conjugate gradient. For such iterative methods, we further note that it is possible to use the previous predictor $M_{t-1}$ as a warm start.

\subsection{Accelerating convolutions}
In the batch setting (or in the online setting, when all the inputs $x_t$ are known in advance), it is easy to see that the convolution components of all feature vectors $\tilde X_t$ can be computed in a single pass, by pointwise multiplication in the Fourier domain. Using the fast Fourier transform, one can implement all convolutions in time $O(nk T \log T)$, nearly linear in the size of the input. This mitigates what would otherwise be a quadratic dependence on $T$. Many software libraries provide an FFT-based implementation of convolution.

\section{Proof of the relaxation theorem}

\label{appendix-section:approximate-relaxation}
In this section, we follow the proof structure given in Section~\ref{subsection:analysis-convex-relaxation}, and conclude Theorem~\ref{thm:main-appx-relaxation}.

Before proceeding, we note here that the algorithm could have used filters of length $T-1$ instead of $T$, obtained from the eigenvectors of $Z_{T-1}$. However, since carrying this $-1$ through the statements and analysis degrades clarity significantly, we use a slightly suboptimal matrix throughout this exposition.

\subsection{Proof of Lemma~\ref{lem:pca-lemma}}
First, we develop a spectral bound for \emph{average} reconstruction error of $\mu(\alpha)$, when $\alpha$ is drawn uniformly from the unit interval $[0,1]$.
This is controlled by the tail eigenvalues of the second moment matrix of $\mu(\alpha)$, just as in PCA:
\begin{lemma}
\label{lem:vanilla-pca}
Let $\{(\sigma_j, \phi_j)\}_{j=1}^T$ be the eigenpairs of $Z$, in decreasing order by eigenvalue.
Let $\Psi_k$ be the linear subspace of $\R^T$ spanned by $\{\phi_1, \ldots, \phi_k\}$. Then,
\[\int_{0}^1 \norm{ \mu(\alpha) - \mathrm{Proj}_{\Psi_k}\!(\alpha) }^2 \,d\alpha \leq \sum_{j = k+1}^T \sigma_{j}. \]
\end{lemma}
\begin{proof}
Let $r(\alpha)$ denote the residual $\mu(\alpha) - \mathrm{Proj}_{\Psi_k}\!(\alpha)$, and let $U_r \in \R^{T \times r}$ whose columns are $\phi_1, \ldots, \phi_r$, so that
\[r(\alpha) = \Pi_r \mu(\alpha) := (I - U_rU_r^\top) \mu(\alpha).\]
Write the eigendecomposition $Z_T = U_T \Sigma U_T^\top$.
Then,
\begin{align*}
\int_0^1 \norm{r(\alpha)}^2 \,d\alpha &= \int_0^1 \Tr(r(\alpha) \otimes r(\alpha)) \,d\alpha = \int_0^1 \Tr \pa{ \Pi_r \mu(\alpha) \mu(\alpha)^\top \Pi_r } \,d\alpha \\
&= \int_0^1 \Tr \pa{ \Pi_r Z \Pi_r } \,d\alpha = \int_0^1 \Tr \pa{ \Pi_r U_T \Sigma U_T^\top \Pi_r } \,d\alpha.
\end{align*}
Noting that $\Pi_r U_T$ is just $U_T$ with the first $r$ columns set to zero, the integrand becomes $\sum_{j=k+1}^T \Sigma_{jj}$, which is the stated bound.
\end{proof}

In fact, this bound \emph{in expectation} turns into a bound for \emph{all} $\alpha$. We show this by noting that $\norm{ r(\alpha) }^2$ is Lipschitz in $\alpha$, so its maximum over $\alpha \in [0,1]$ cannot be too much larger than its mean. We state and prove this here:
\begin{lemma}
\label{lem:super-pca}
For all $\alpha \in [0, 1]$, it holds that
\[\norm{r(\alpha)}^2 \leq \sqrt{6 \sum_{j = k+1}^T \sigma_{j} }.\]
\end{lemma}
\begin{proof}
By part (ii) of Lemma~\ref{lem:mu-l2}, $\norm{\mu(\alpha)}^2$ is $3$-Lipschitz; since $\Pi_r$ is contractive,
$\norm{r(\alpha)}^2$ is also $3$-Lipschitz. Now, let $R := \max_{0 \leq \alpha \leq 1} \norm{r(\alpha)}^2$. Notice that $R \leq \max_{0 \leq \alpha \leq 1} \norm{\mu(\alpha)}^2 \leq 1$, by part (i) of Lemma~\ref{lem:mu-l2}. Subject to achieving a maximum at $R$, the non-negative $3$-Lipschitz function $g : [0,1] \rightarrow \R$ with the smallest mean is given by the triangle-shaped function
\[\Delta(\alpha) = \max(R-3\alpha, 0),\]
for which
\[\int_0^1 \Delta(\alpha) \,d\alpha = R^2 / 6.\]
In other words,
\[R^2 / 6 \leq  \int_0^1 \norm{r(\alpha)}^2 \,d\alpha.\]
But Lemma~\ref{lem:vanilla-pca} gives a bound on the RHS, so we conclude
\[\max_{\alpha \in [0,1]} \norm{r(\alpha)}^2 \leq R \leq \sqrt{6 \sum_{j = k+1}^T \sigma_{j} },\]
as desired. The stated upper bound on this quantity comes a bound of this spectral tail of the Hankel matrix $Z_T$ (see Lemmas~\ref{lem:Z-decay} and \ref{lem:Z-sum-decay}); this completes the proof of Lemma~\ref{lem:pca-lemma}.
\end{proof}

\subsection{Proof of Theorem~\ref{thm:main-appx-relaxation}}
It remains to apply Lemma~\ref{lem:pca-lemma} to the original setting, which will complete the low-rank approximation result of Theorem~\ref{thm:main-appx-relaxation}.
Indeed, following Section~\ref{subsection:analysis-convex-relaxation}, we have
\begin{align*}
\zeta_t \defeq M_\Theta \tilde X_t - \hat y_t = \sum_{l=1}^d (c_l \otimes b_l) \sum_{i=1}^{T-1} \bra{ \tilde\mu(\alpha_l) - \mu(\alpha_l) }\!(i) \cdot x_{t-i}.
\end{align*}
View each of the $n$ coordinates in the inner summation as an inner product between the
length-$T$ sequence $\tilde\mu(\alpha_l) - \mu(\alpha_l)$ and coordinates
$X(j) := (x_1(j), \ldots, x_T(j))$, which are entrywise bounded by $R_x$. Then,
by H\"older's inequality and Lemma~\ref{lem:pca-lemma}, we know that this inner product
has absolute value at most
\[ \norm{X(j)}_\infty \norm{\tilde\mu(\alpha_l) - \mu(\alpha_l)}_1 \leq
\norm{X(j)}_\infty \cdot \sqrt{T} \norm{\tilde\mu(\alpha_l) - \mu(\alpha_l)}_2 \leq
O\pa{ R_x \sqrt{T} \cdot c_1^{-k/\log T} \log^{1/4} T },
\]
with $c_1 = \sqrt{c_0}$. There are $n$ such coordinates,
so this inner summation is a vector with $\ell_2$
norm at most
\[ O\pa{ R_x \sqrt{nT} \cdot c_1^{-k/\log T} \log^{1/4} T }. \]

Thus, in all, we have
\begin{align*}
\norm{\zeta_t}_2 &\leq O\pa{ \norm{B}_F \norm{C}_F R_x \sqrt{nT} \cdot c_1^{-k/\log T} \log^{1/4} T }.
\end{align*}

In summary, we have shown that for every system
$\Theta$ from which a predictor for the discrete derivative of the LDS arises,
there is some $M_\Theta$ whose predictions are pointwise $\norm{\zeta_t}_2$-close.
This residual bound can be driven down exponentially by increasing the number of filters $k$.

Finally, to get an inequality on the \emph{total} squared error, we compute
\begin{align}
\sum_{t=1}^T \norm{M_{\Theta} \tilde X_t - y_t}^2 &= \sum_{t=1}^T \norm{\hat y_t - y_t + \zeta_t}^2
\leq \sum_{t=1}^T \pa{ \norm{\hat y_t - y_t}^2 + \norm{\zeta_t}^2 + 2 \norm{\hat y_t - y_t} \, \norm{\zeta_t} } \nonumber \\
&\leq  \sum_{t=1}^T \norm{\hat y_t - y_t}^2 \;\; + O\pa{ (R_\Theta^4 R_x^2 L_y^2 k) \, T^{3/2} n^{1/2} \cdot c_1^{-k/\log T} \log^{1/4} T }, \label{eq:errbound} \\
&\leq  \sum_{t=1}^T \norm{\hat y_t - y_t}^2 \;\; + O\pa{ R_\Theta^4 R_x^2 L_y^2 \, T^{5/2} n^{1/2} \cdot c_1^{-k/\log T} \log^{1/4} T } \nonumber,
\end{align}
where inequality (\ref{eq:errbound}) invokes Corollary~\ref{cor:residuals-small}.
 Thus, in all, it suffices to choose
\[ \frac{k}{\log T} \geq \Omega\pa{ \log \frac{R_\Theta R_x L_y\, n T}{\eps} } \]
to force the $O(\cdot)$ term to be less than $\eps$,
noting that the powers of $n$ and $T$ show up as a constant factor in front of the $\log(\cdot)$.
This completes the proof. \qed

\section{Proof of the main regret bound}

\label{appendix-section:analysis-low-regret}
In this part of the appendix, we follow the proof structure outlined Section~\ref{subsection:analysis-low-regret}, to establish Theorem~\ref{thm:main-online}.
The lemmas involved also appear in the proof of the batch variant (Theorem~\ref{thm:main-batch}).

\subsection{Diameter bound: controlling the comparator matrix}
\label{appendix-subsection:diam}

We will show that the $M_\Theta$ that competes with a system $\Theta$ is not too much larger than $\Theta$, justifying the choice of $R_M = \Omega\pa{ R_\Theta^2 \sqrt{k} }$.
Of course, this implies that the diameter term in the regret bound is
$D_\mathrm{max} = 2R_M$. Concretely:
\begin{lemma}
\label{lem:wheres-the-matrix}
For any LDS parameters $\Theta = (A,B,C,D,h_0 = 0)$ with $0 \preccurlyeq A \preccurlyeq I$ and $\norm{B}_F, \norm{C}_F, \norm{D}_F, \norm{h_0} \leq R_\Theta$, the corresponding matrix $M_\Theta \in \hat\Hcal$ (which realizes the relaxation in Theorem~\ref{thm:main-appx-relaxation}) satisfies
\[ \norm{M_{\Theta}}_F^2 \leq O\pa{ R_\Theta^2 \sqrt{k} }. \]
\end{lemma}
\begin{proof}
Recalling our construction $M_{\Theta}$ in the proof of Theorem~\ref{thm:main-appx-relaxation}, we have
\begin{itemize}
\item $\norm{ M^{(j)} }_F \leq \norm{B}_F\norm{C}_F \cdot \max_{\ell \in [d]} \sigma_j^{-1/4} \ang{\phi_j, \mu(\alpha_l)} $, for each $1 \leq j \leq k$.
\item $\norm{ M^{(x')} }_F = \norm{D}_F \leq O(R_\Theta)$.
\item $\norm{ M^{(x)} }_F \leq \norm{B}_F\norm{C}_F + \norm{D}_F \leq O(R_\Theta^2)$.
\end{itemize}
Recall that we do not consider $M^{(y)}$ as part of the online learning
algorithm; it is always the identity matrix. Thus, for the purposes of this analysis,
it does not factor into regret bounds.

In Lemma~\ref{lem:m-small}, we show that the reconstruction coefficients $\sigma_j^{-1/4} \ang{\phi_j, \mu(\alpha_l)}$
are bounded by an absolute constant; thus, those matrices each have Frobenius
norm at most $O(R_\Theta^2)$. These terms dominate the Frobenius norm of the entire
matrix, concluding the lemma.
\end{proof}

This has a very useful consequence:
\begin{corollary}
\label{cor:residuals-small}
The predictions $\hat y_t = M\tilde X_t$ made by choosing $M$
such that $\norm{M}_F \leq O(R_\Theta^2 \sqrt{k})$ satisfy
\[ \norm{ \hat y_t - y_t }^2 \leq O(R_\Theta^4 R_x^2 L_y^2 k). \]
\end{corollary}

\subsection{Gradient bound and final details}
\label{appendix-subsection:diameter-gradient}
A subtle issue remains: the gradients may be large,
as they depend on $\tilde X_t$, defined by convolutions of the entire input
time series by some filters $\phi_j$. Note that these filters do \emph{not}
preserve mass: they are $\ell_2$ unit vectors, which may cause the norm of
the part of $\tilde X_t$ corresponding to each filter to be as large as $\sqrt{T}$.

Fortunately, this is not the case. Indeed, we have:
\begin{lemma}
\label{lem:xfilt-small}
Let $\{(\sigma_j, \phi_j)\}_{j=1}^T$ be the eigenpairs of $Z$, in decreasing order by eigenvalue.
Then, for each $1 \leq j,t \leq T$, it holds that
\[ \norm{ \sigma^{1/4} (\phi_j * X)_t }_\infty \leq O\pa{ R_x \log T }. \]
\end{lemma}
\begin{proof}
Each coordinate of $(\sigma^{1/4} \phi_j * X)_t$ is the inner product between $\phi_j$ and
a sequence of $T$ real numbers, entrywise bounded by $\sigma_j^{1/4} R_x$.
Corollary~\ref{cor:phi-l1} shows that this is at most $O(\log T)$,
a somewhat delicate result which uses matrix perturbation.
\end{proof}
Thus, $\tilde X_t$ has $nk$ entries with absolute value bounded by $O\pa{ R_x \log T }$,
concatenated with $x_t$ and $x_{t-1}$. So, we have:
\begin{corollary}
\label{lem:feats-small}
Let $\tilde X_t$ be defined as in Algorithm~\ref{alg:wave-ogd}, without the $y_{t-1}$ portion.
Then,
\[ \norm{\tilde X_t}_2 \leq O\pa{ R_x \log T \sqrt{nk} }. \]
\end{corollary}

Our bound on the gradient follows:
\begin{lemma}
\label{lem:grad-small}
Suppose $\Mcal$ is chosen with diameter $O(R_\Theta^2)$. Then, the gradients satisfy
\begin{align*}
G_\mathrm{max} \defeq \max_{\substack{ M \in \Mcal, \\ 1 \leq t \leq T }} \norm{ \nabla f_t(M) }_F \leq O \pa{ R_\Theta^2 R_x^2 L_y \cdot nk^{3/2} \log^2 T }.
\end{align*}
\end{lemma}
\begin{proof}
We compute the gradient, and apply Lemma~\ref{lem:xfilt-small}:
\begin{align*}
\nabla f_t(M) = \nabla \pa{ \norm{y_t - M \tilde X_t}^2 } = 2 (M \tilde X_t - y) \otimes \tilde X_t,
\end{align*}
so that
\begin{align*}
\norm{ \nabla f_t(M) }_F &= 2 \norm{M \tilde X_t - y_t}_2 \cdot \norm{\tilde X_t}_2 \\
&\leq 2 \pa{ \norm{M}_F \norm{\tilde X_t}_2 + L_y } \norm{\tilde X_t}_2 \\
&\leq 2 \pa{ \pa{ R_\Theta^2 \sqrt{k} }\pa{ R_x \log T \sqrt{nk} } + L_y } \pa{ R_x \log T \sqrt{nk} } \\
&\leq O \pa{ R_\Theta^2 R_x^2 L_y \cdot nk^{3/2} \log^2 T },
\end{align*}
as desired.
\end{proof}

\subsection{Assembling the regret bound}
Using Lemma~\ref{thm:ogd-textbook-bound},
collecting all terms from Lemmas~\ref{lem:wheres-the-matrix} and \ref{lem:grad-small}, we have in summary
\begin{align*}
D_\mathrm{max} G_\mathrm{max} &= O\pa{ R_\Theta^2 \sqrt{k} } \cdot O \pa{ R_\Theta^2 R_x^2 L_y \cdot nk^{3/2} \log^2 T } \\
&= O\pa{ R_\Theta^4 R_x^2 L_y nk^2 \log^2 T }.
\end{align*}

To compete with systems with parameters bounded by $R_\Theta$,
in light of Theorem~\ref{thm:main-appx-relaxation}, $k$ should
be chosen to be $\Theta\pa{ \log^2 T \log (R_x L_y R_\Theta n) }$. It suffices to set the relaxation approximation error $\eps$ to be a constant; in the online case, this is not the bottleneck of the regret bound. In all, the regret bound from online gradient descent is
\[ \Regret(T) \leq O\pa{ R_\Theta^4 \, R_x^2 \, L_y \, \log^2 (R_\Theta R_x L_y n) \cdot n \sqrt{T} \log^6 T }, \]
as claimed. \qedsymbol

\subsection{Diminishing effect of the initial hidden state}
\label{appendix-subsection:hidden-state}
Finally, we show that $h_0$ is not significant in this online setting,
thereby proving a slightly more general result.
Throughout the rest of the analysis, we considered the comparator $\Theta^*$,
which forces the initial hidden state to be the zero vector. We will show that
this does not make much worse predictions than $\Theta^{**}$, which is
allowed to set $\norm{h_0}_2 \leq R_\Theta$. We quantify this below:

\begin{lemma}
Relaxing the condition $h_0 = 0$ for the comparator in Theorem~\ref{thm:main-online}
increases the regret (additively) by at most
\[ O\pa{ R_\Theta^4 R_x L_y \, \log(R_\Theta R_x L_y n) \, \log^2 T }. \]
\end{lemma}
\begin{proof}
First, an intuitive sketch:
Lemma~\ref{lem:mu-envelope} states that for any $\alpha$, there is an ``envelope'' bound
$\mu(\alpha)(t) \leq \frac{1}{t+1}$. This means that the influence of $h_0$ on the derivative
of the impulse response function decays as $1/t$; thus, we can expect the total ``loss of expressiveness'' caused by forcing $h_0 = 0$ to be only logarithmic in $T$.

Indeed, with a nonzero initial hidden state, we have
\begin{align*}
\hat y_t - y_{t-1} &= (CB + D)x_t - Dx_{t-1} + \sum_{i=1}^{T-1} C(A^i - A^{i-1})Bx_{t-i} + C(A^t - A^{t-1}) h_0.
\end{align*}

Let $\hat y_1, \ldots, \hat y_T$ denote the predictions made by an LDS $\Theta^{**} = (A,B,C,D,h_0)$ whose;
$\hat y^\emptyset_1, \ldots, \hat y^\emptyset_T$ denote the predictions made by the LDS
with the same $(A,B,C,D)$, but $h_0$ set to 0. Then, we have
\begin{align*}
\norm{ \hat y_t - \hat y^\emptyset_t } &= \norm{ C (A^t - A^{t-1}) h_0 }
= \big\lVert \sum_{l=1}^d C \bra{ \mu(\alpha_l)(t) \cdot e_l \otimes e_l } h_0 \big\rVert \\
&\leq \frac{\norm{C}_F \norm{h_0} \sqrt{n}}{t} \leq \frac{ R_\Theta^2 \sqrt{n} }{ t }.
\end{align*}
Thus we have, for vectors $u_t$ satisfying $\norm{u_t} \leq R_\Theta^2 / t$:
\begin{align*}
\sum_{t=1}^T \norm{ \hat y^\emptyset_t - y_t }^2
&= \sum_{t=1}^T \norm{ \hat y_t + u_t - y_t }^2
\leq \sum_{t=1}^T \norm{ \hat y_t - y_t }^2 + \norm{u_t}^2 + 2 \norm{\hat y_t - y_t} \, \norm{u_t} \\
&\leq \sum_{t=1}^T \norm{ \hat y_t - y_t }^2 + O\pa{ R_\Theta^4 n }
+ O\pa{ (R_\Theta^2 R_x L_y \sqrt{k}) \cdot R_\Theta^2 \sqrt{n} \log T } \\
&\leq \sum_{t=1}^T \norm{ \hat y_t - y_t }^2
+ O\pa{ R_\Theta^4 R_x L_y \, \log(R_\Theta R_x L_y n) \, n \log^2 T },
\end{align*}
where the inequalities respectively come from Cauchy-Schwarz, Lemma~\ref{lem:mu-envelope}, and Lemma~\ref{cor:residuals-small}. This completes the proof.
\end{proof}

Thus, strengthening the comparator by allowing a nonzero $h_0$ does not improve the asymptotic regret bound from Theorem~\ref{thm:main-online}.

\section{Properties of the Hankel matrix $Z_T$}

\label{appendix-section:hankel-properties}
In this section, we show some technical lemmas about the family of Hankel matrices $Z_T$,
whose entries are given by
\[ Z_{ij} = \frac{2}{(i+j)^3 - (i+j)}.\]

\subsection{Spectral tail bounds}

We use the following low-approximate rank property of positive semidefinite Hankel matrices, from \cite{beckermann2016singular}:
\begin{lemma}[Cor.~5.4 in \cite{beckermann2016singular}]
\label{lem:magic-spectrum}
Let $H_n$ be a psd Hankel matrix of dimension $n$. Then,
\[\sigma_{j+2k}(H_n) \leq 16 \bra{ \exp\pa{ \frac{\pi^2}{4\log(8\lfloor n/2 \rfloor/\pi)} } }^{-2k+2} \sigma_j (H_n). \]
\end{lemma}

Note that the Hankel matrix $Z_T$ is indeed positive semidefinite, because we constructed it as
\[Z = \int_0^1 \mu(\alpha) \otimes \mu(\alpha) \,d\alpha, \]
for a certain $\mu(\alpha) \in \R^T$. Also, note that at no point do we rely upon $Z_T$ being positive definite or having all distinct eigenvalues, although both seem to be true.

The first result we need is an exponential decay of the tail spectrum of $Z$.
\begin{lemma}
\label{lem:Z-decay}
Let $\sigma_j$ be the $j$-th top singular value of $Z := Z_T$. Then, for all $T \geq 10$, we have
\[\sigma_j \leq \min\pa{ \frac{3}{4}, K \cdot c^{-j / \log T} }, \]
where $c = e^{\pi^2 / 4} \approx 11.79$, and $K < 10^6$ is an absolute constant.
\end{lemma}
\begin{proof}
We begin by noting that for all $j$, $\sigma_j \leq \Tr(Z) = \sum_{i=1}^T \frac{1}{(2i)^3 - 2i} < \sum_{i=1}^\infty \frac{1}{4i^3} < \frac{3}{4}$.

Now, since $T \geq 10$ implies $8\lfloor T/2 \rfloor / \pi > T$, we have
\begin{align*}
\sigma_{2 + 2k} \leq \sigma_{1 + 2k} &< 12 \cdot \bra{ \exp\pa{ \frac{\pi^2}{2 \log T} } }^{-k+1} \\
&< 1680 \cdot c^{-2k / \log T}.
\end{align*}

Thus, we have that for all $j$,
\begin{align*}
\sigma_j < 1680 \cdot c^{-(j - 2) / \log T} < 235200 \cdot c^{-j / \log T},
\end{align*}
completing the proof.
\end{proof}

We also need a slightly stronger claim: that all spectral gaps are large. Lemma~\ref{lem:Z-decay} does not preclude that there are closely clustered eigenvalues under the exponential tail bound. In fact, this cannot be the case:
\begin{lemma}
\label{lem:Z-sum-decay}
Let $\sigma_j$ be the $j$-th top singular value of $Z := Z_T$. Then, if $T \geq 60$, we have
\[\sum_{j' > j} \sigma_{j'} < 400 \log T \cdot \sigma_j. \]
\end{lemma}
\begin{proof}
For convenience, define $\sigma_{j} := 0$ when $j \geq T$. Picking $k=4$ and using Lemma~\ref{lem:magic-spectrum}, we have that
\begin{align*}
\beta_j &:= \sum_{q=1}^T \sigma_{j+4q} < 16 \sigma_j \sum_{q=1}^\infty \bra{ \exp\pa{ \frac{-\pi^4}{4\log T} } }^{q} = 16\sigma_j \cdot \frac{1}{1 - \exp\pa{\frac{ -\pi^4 }{ 4 \log T }} } \\
&< 100 \log T \cdot \sigma_j,
\end{align*}
where the last inequality follows from the fact that
\[\frac{1}{1 - e^{-x}} < \frac{6}{x}\]
whenever $x < 6$, and setting $x := \frac{-\pi^4}{4 \log T} \leq \frac{-\pi^4}{4 \log 60} < 6.$

Thus, we have
\[\sum_{j' > j} \sigma_{j'} = \beta_j + \beta_{j+1} + \beta_{j+2} + \beta_{j+3} < 4\beta_j < 400 \log T \cdot \sigma_j,\]
as desired.
\end{proof}

\subsection{Decaying reconstruction coefficients}

To show a bound on the entries of $M_\Theta$, we need the
following property of $Z_T$:

\begin{lemma}
\label{lem:m-small}
For any $0 \leq \alpha \leq 1$ and $1 \leq j \leq T$, we have
\[ \left\lvert\ang{\phi_j, \mu(\alpha)}\right\rvert \leq 6^{1/4} \, \sigma_{j}^{1/4}. \]
\end{lemma}
\begin{proof}
We have
\begin{align*}
\int_0^1 \ang{\phi_j, \mu(\alpha)}^2 \, d\alpha &= \int_0^1 \phi_j^T \pa{ \mu(\alpha) \otimes \mu(\alpha) } \phi_j \\
&= \phi_j^T Z_T \phi_j = \sigma_j.
\end{align*}

Thus, we have a bound on the expectation of the squared coefficient, when $\alpha$ is drawn uniformly from $[0,1]$.
We proceed with the same argument as was used to prove Lemma~\ref{lem:super-pca}: since $\norm{ \mu(\alpha) }^2$ is 3-Lipschitz in $\alpha$, so is $\ang{\phi_j, \mu(\alpha)}^2$ (since projection onto the one-dimensional subspace spanned by $\phi_j$ is contractive). Thus it holds that
\[ \max_{\alpha \in [0,1]} \ang{\phi_j, \mu(\alpha)}^2 \leq \sqrt{ 6 \sigma_j }, \]
from which the claim follows.
\end{proof}

\subsection{Controlling the $\ell_1$ norms of filters}

To bound the size of the convolutions,
we need to control the $\ell_1$ norm of the eigenvectors $\phi_j$
with a tighter bound than $\sqrt{T}$.
Actually, we prove a more general result,
bounding the $\ell_2 \rightarrow \ell_1$ subordinate norm of $Z^{1/4}$:

\begin{lemma}
\label{lem:mat-perturb-quarter}
Let $Z := Z_T$.
Then, for every $T > 0$, and $v \in \R^n$ with $\norm{v}_2 = 1$, we have
\[ \norm{Z^{1/4} v}_1 \leq 2 + 2\log_2 T. \]
\end{lemma}
\begin{proof}
We take the following steps:
\begin{enumerate}
\item[(i)] Start with a constant $T_0$; the subordinate norm of $Z_{T_0}$ is clearly bounded by a constant.
\item[(ii)] Argue that doubling the size of the matrix ($T \mapsto 2T$) comprises only a small perturbation,
which will only affect the eigenvalues of the matrix by a small amount. This will show up in the subordinate norm
as an additive constant.
\item[(iii)] Iterate the doubling argument $O(\log T)$ times to reach $Z_T$ from $Z_{T_0}$, to conclude the lemma.
\end{enumerate}

The only nontrivial step is (ii), which we prove first.
Consider the doubling step from $T$ to $2T$.
Let $Z$ denote the $2T$-by-$2T$ matrix which has $Z_T$ as its upper left $T$-by-$T$ submatrix,
and zero everywhere else. Let $Z'$ denote $Z_{2T}$, and call $E = Z' - Z$, which we interpret as the
matrix perturbation associated with doubling the size of the Hankel matrix.

Notice that when $T \geq 2$,
$E$ is entrywise bounded by $\frac{2}{(T+2)^3 - (T+2)} \leq \frac{2}{T^3}$,
which we call $e_\mathrm{max}$ for short.
Then, $\norm{E}_\mathrm{op}$ is at most $Te_\mathrm{max} \leq \frac{2}{T^2}$.

Hence, by the generalized Mirsky inequality of \cite{audenaert2014generalisation} (setting $f(x) = x^{1/4}$),
we have a bound on how much $E$ perturbs the fourth root of $Z$:
\[ \norm{Z^{1/4} - Z'^{1/4}}_2
\leq \norm{E}_2^{1/4} \leq \pa{ \frac{2}{T^2} }^{1/4} < \frac{2}{\sqrt{T}}. \]

Thus we have
\begin{align*}
\norm{ Z'^{1/4} }_{2 \rightarrow 1} &\leq \norm{ Z^{1/4} }_{2 \rightarrow 1} + \norm{Z^{1/4} - Z'^{1/4}}_{2 \rightarrow 1} \\
&\leq \norm{ Z^{1/4} }_{2 \rightarrow 1} + \sqrt{T} \cdot \norm{Z^{1/4} - Z'^{1/4}}_{2} \\
&\leq \norm{ Z^{1/4} }_{2 \rightarrow 1} + \sqrt{T} \cdot \frac{2}{\sqrt{T}} \\
&= \norm{ Z^{1/4} }_{2 \rightarrow 1} + 2.
\end{align*}
Thus, doubling the dimension increases the subordinate norm by at most a constant.
We finish the argument: start at $T_0 = 2$, for which it clearly holds that
\[ \norm{ Z^{1/4}_2 }_{2 \rightarrow 1} < \sqrt{2} \norm{ Z_2^{1/4} }_F < \sqrt{2} \norm{ Z_4 }_F < 2. \]

Noting that the norm is clearly monotonic in $T$, we repeat the doubling argument $\lfloor \log_2 T \rfloor$ times,
so that
\[ \norm{ Z^{1/4}_T }_{2 \rightarrow 1} \leq \norm{ Z^{1/4}_{2 \cdot 2^{\lfloor \log_2 T \rfloor}} }_{2 \rightarrow 1} < \norm{ Z^{1/4}_2 }_{2 \rightarrow 1} + 2 \lfloor \log_2 T \rfloor
< 2 + 2 \log_2 T, \]
as claimed.
\end{proof}

We give an alternate form here:
\begin{corollary}
\label{cor:phi-l1}
Let $(\sigma_j, \phi_j)$ be the $j$-th largest eigenvalue-eigenvector pair of $Z$. Then,
\[ \norm{\phi_j}_1 \leq O\pa{ \frac{\log T}{\sigma_j^{1/4}} }. \]
\end{corollary}

\section{Properties of $\mu(\alpha)$}

\label{appendix-section:moment-curve}

Throughout this section, fix some $T \geq 1$; then, recall that $\mu(\alpha) \in \R^T$ is defined as the vector whose $i$-th entry is $(1-\alpha)\alpha^{i-1}$. At various points, we will require some elementary properties of $\mu(\alpha)$, which we verify here.

\begin{lemma}[$1/t$ envelope of $\mu$]
\label{lem:mu-envelope}
For any $t \geq 0$ and $0 \leq \alpha \leq 1$, it holds that
\[(1-\alpha)\alpha^t \leq \frac{1}{t+1}.\]
\end{lemma}
\begin{proof}
Setting the derivative to zero, the global maximum occurs at $\alpha^* = \frac{t}{t+1}$. Thus,
\begin{align*}
(1-\alpha^*)(\alpha^*)^t = \frac{1}{t+1} \pa{1 - \frac{1}{t+1}}^{t} \leq \frac{1}{t+1},
\end{align*}
as claimed.
\end{proof}
\begin{corollary}
\label{lem:mu-log}
Let $T \geq 1$. For $t = 1, \ldots, T$, let $\alpha_t \in [0, 1]$ be different in general. Then,
\[ \sum_{t=1}^T (1-\alpha_t)\alpha_t^{t-1} \leq H_n = O(\log T), \]
where $H_n$ denotes the $n$-th harmonic number.
\end{corollary}
\begin{lemma}[$\ell_1$-norm is small]
\label{lem:mu-l1}
For all $T \geq 1$ and $0 \leq \alpha \leq 1$, we have
\[\norm{\mu(\alpha)}_1 \leq 1.\]
\end{lemma}
\begin{proof}
We have
\[ \norm{\mu(\alpha)}_1 = (1-\alpha) \sum_{t=1}^T \alpha^{t-1} \leq (1-\alpha) \sum_{t=1}^\infty \alpha^{t-1} = 1,\]
proving the claim.
\end{proof}
\begin{lemma}[$\ell_2$-norm is small and Lipschitz]
\label{lem:mu-l2}
For all $T \geq 1$ and $0 \leq \alpha \leq 1$, we have
\begin{enumerate}
\item[(i)] $\norm{\mu(\alpha)}^2 \leq 1$.
\item[(ii)] $\abs{ \frac{d}{d\alpha} \norm{\mu(\alpha)}^2 } \leq 3$.
\end{enumerate}
\end{lemma}
\begin{proof}
For the first inequality, compute
\begin{align*}
\norm{\mu(\alpha)}^2 &= \sum_{i=1}^T \pa{ (\alpha-1) \alpha^{i-1} }^2
=  \sum_{i=1}^T \alpha^{2i} - 2\alpha^{2i-1} + \alpha^{2i-2} \\
&= \frac{(\alpha^2 - 2\alpha + 1)(1 - \alpha^{2T})}{1 - \alpha^2} = \frac{(1 - \alpha)(1 - \alpha^{2T})}{1 + \alpha} \leq 1.
\end{align*}

For the second, differentiate the closed form to obtain
\begin{align*}
\abs{ \frac{d}{d\alpha} \norm{\mu(\alpha)}^2 } &= \abs{ \frac{2(\alpha^T - 1) + T\alpha^{T-1}(\alpha^2 - 1)}{(1+\alpha)^2} }
\leq \frac{ 2(1 - \alpha^T) + T\alpha^{T-1}(1 - \alpha^2) }{ (1+\alpha)^2 } \\
&= \frac{2 - \alpha^T}{(1+\alpha)^2} + \frac{T\alpha^{T-1}(1-\alpha)}{1+\alpha}
\leq 2 + T\alpha^{T-1}(1-\alpha) \leq 3,
\end{align*}
where the final inequality uses Lemma~\ref{lem:mu-envelope}.
\end{proof}

\subsection{The Lipschitzness of a true LDS}
We claim in Section~\ref{subsection:prelim-oco} that $L_y$, the Lipschitz constant
of a true LDS, is bounded by $\norm{B}_F \norm{C}_F R_x$.
We now prove this fact, which is a consequence of the above facts.

\begin{lemma}
\label{lem:true-lds-lipschitz}
Let $\Theta = (A,B,C,D,h_0)$ be a true LDS, which produces outputs $y_1, \ldots, y_T$ from inputs $x_1, \ldots, x_T$
by the definition in the recurrence, without noise. Let $0 \preccurlyeq A \preccurlyeq I$, and $\norm{B}_F,\norm{C}_F,\norm{D}_F,\norm{h_0} \leq R_\Theta$. Then, we have that for all $t$,
\[ \norm{y_t - y_{t-1}} \leq O(R_\Theta^2 R_x). \]
\end{lemma}
\begin{proof}
We have that for all $1 \leq t \leq T$,
\begin{align*}
\norm{ y_t - y_{t-1} } &= \bigg\lVert (CB + D)x_t - Dx_{t-1} + \sum_{i=1}^{T-1} C(A^i - A^{i-1})Bx_{t-i} + C(A^t - A^{t-1}) h_0 \bigg\rVert \\
&\leq (\norm{B}_F\norm{C}_F + 2\norm{D}_F) R_x  +  \norm{B}_F\norm{C}_F R_x + \frac{ \norm{C}_F \norm{h_0} }{t},
\end{align*}
where the inequality on the second term arises from Lemma~\ref{lem:mu-l1} and
the inequality on the third from Lemma~\ref{lem:mu-log}. This implies the lemma.
\end{proof}

\end{document}